\documentclass[letterpaper,twocolumn,10pt, dvipsnames]{article}
\usepackage{usenix2019_v3}

\usepackage{times}
\usepackage{epsfig}
\usepackage{graphicx}
\usepackage{amsmath}
\usepackage{amssymb}
\usepackage{amsthm}

\usepackage{algorithm}
\usepackage{algorithmic}

\newcommand{\greenup}{\textcolor{Green}{$\uparrow$}}

\usepackage[]{xcolor}
\usepackage{sectsty}
\usepackage{lipsum}%

\usepackage{amsmath}
\usepackage{amsfonts}       %
\usepackage{booktabs}       %
\usepackage{multirow}
\usepackage{multicol}       %
\usepackage{caption}
\usepackage{subcaption}

\newcommand{\argmax}{\mathop{\mathrm{argmax}}}
\newcommand{\set}[1]{\left\{ #1 \right\}}
\newcommand{\serr}{\widehat{\gerr}}
\newcommand{\gerr}{\operatorname{err}}

\newcommand{\beq}{\begin{eqnarray*}}
\newcommand{\eeq}{\end{eqnarray*}}
\newcommand{\beqn}{\begin{eqnarray}}
\newcommand{\eeqn}{\end{eqnarray}}
\newcommand{\ben}{\begin{enumerate}}
\newcommand{\een}{\end{enumerate}}
\newcommand{\bit}{\begin{itemize}}
\newcommand{\eit}{\end{itemize}}
\newcommand{\paren}[1]{\left( #1 \right)}
\newcommand{\R}{\mathbb{R}}
\newcommand{\E}{\mathop{\mathbb{E}}}

\renewcommand{\P}{\mathop{\mathbb{P}}}

\newcommand{\etal}{\textit{et al}.}

\newcommand{\hide}[1]{}

\newtheorem{theorem}{Theorem}

\usepackage{float}

\usepackage{xspace}
\newcommand\mycolor{\color{black}\xspace}
\newcommand\mycolorgreen{\color{black}\xspace}

\usepackage[capitalize]{cleveref}
\crefformat{section}{\S#2#1#3} %
\crefformat{subsection}{\S#2#1#3}
\crefformat{subsubsection}{\S#2#1#3}

\Crefname{table}{Table}{Tables}
\crefname{table}{Tab.}{Tabs.}

\usepackage{tikz}
\usepackage{amsmath}

\usepackage{fancyhdr}

\begin{document}

\date{}

\fancyhf{}
\cfoot{\thepage}

\lhead{Accepted to the 33rd USENIX Security Symposium}

\title{\Large \bf Splitting the Difference on Adversarial Training}

\author{
{\rm Matan Levi}\\
Ben-Gurion University of the Negev\\
matanle@post.bgu.ac.il
\and
{\rm Aryeh Kontorovich}\\
Ben-Gurion University of the Negev\\
karyeh@bgu.ac.il
} %

\maketitle
\thispagestyle{fancy}
\newpage
\begin{abstract}

The existence of adversarial examples points to a basic weakness of deep neural networks. One of the most effective defenses against such examples, adversarial training, entails training models with some degree of robustness, usually at the expense of a degraded natural accuracy. Most adversarial training methods aim to learn a model that finds, for each class, a common decision boundary encompassing both the clean and perturbed examples. In this work, we take a fundamentally different approach by treating the perturbed examples of each class as a separate class to be learned, effectively splitting each class into two classes: ``clean'' and ``adversarial.'' This split doubles the number of classes to be learned, but at the same time considerably simplifies the decision boundaries. We provide a theoretical plausibility argument that sheds some light on the conditions under which our approach can be expected to be beneficial. Likewise, we empirically demonstrate that our method learns robust models while attaining optimal or near-optimal natural accuracy, e.g., on CIFAR-10 we obtain near-optimal natural accuracy of $95.01\%$ alongside significant robustness across multiple tasks. The ability to achieve such near-optimal natural accuracy, while maintaining a significant level of robustness, makes our method applicable to real-world applications where natural accuracy is at a premium. 
As a whole, our main contribution is a general method that confers a significant level of robustness upon classifiers with only minor or negligible degradation of their natural accuracy.


\end{abstract}

\section{Introduction}

Despite their success in a wide variety of challenging tasks, Neural Networks are brittle when faced with
small, imperceptible perturbations to their input; these are commonly referred to as \textit{adversarial examples}, which 
will, with high probability, alter the neural network's classification \cite{szegedy2013intriguing, goodfellow2014explaining, tabacof2016exploring, kurakin2016adversarial, moosavi2016deepfool, carlini2017towards, tramer2017ensemble, carlini2017adversarial, dong2018boosting, xie2019improving, rony2019decoupling}. 
Early 
methods for defense 
against such attacks
were soon broken by
stronger adversaries \cite{athalye2018obfuscated}; 
subsequently, 
{\em adversarial training}
emerged as 
one of
the most effective defenses
\cite{szegedy2013intriguing, goodfellow2014explaining, madry2017towards, zhang2019theoretically}.
These adversarial training techniques aim to learn robust models by solving a min-max optimization problem.

While the inner maximization searches for worst-case adversarial examples during training, and then augments the training data with them, the outer minimization optimizes across model parameters given 
natural and adv. examples.

\begin{figure*}
  \centering
  \begin{subfigure}{0.14\linewidth}
    \label{fig:dbat_1_new}
    \includegraphics[width=0.97\linewidth]{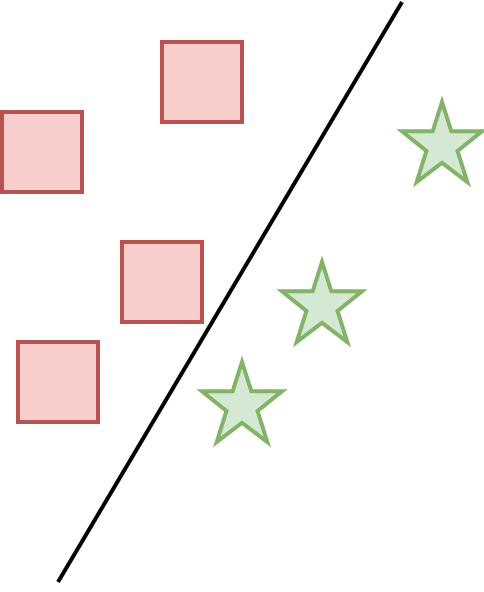}
    \caption{}
  \end{subfigure}
  \hfill
  \begin{subfigure}{0.19\linewidth}
    \label{fig:dbat_2_new}
    \includegraphics[width=0.97\linewidth]{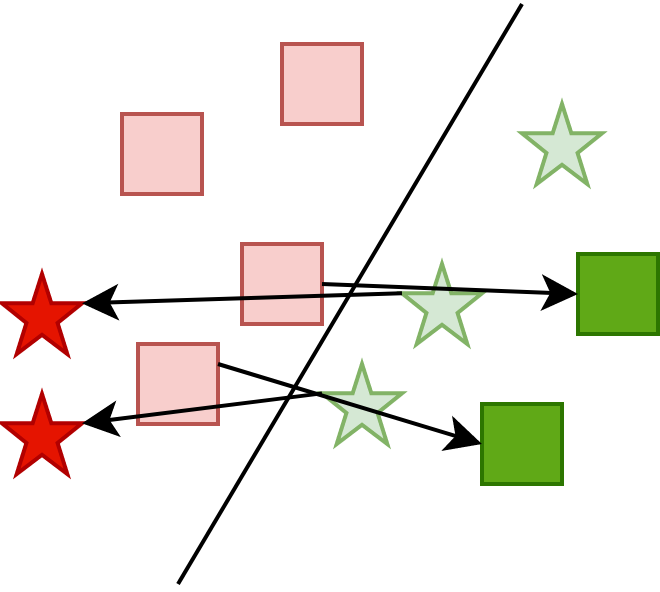}
    \caption{}
  \end{subfigure}
  \hfill
  \begin{subfigure}{0.26\linewidth}
    \label{fig:dbat_3_new}
    \includegraphics[width=0.97\linewidth]{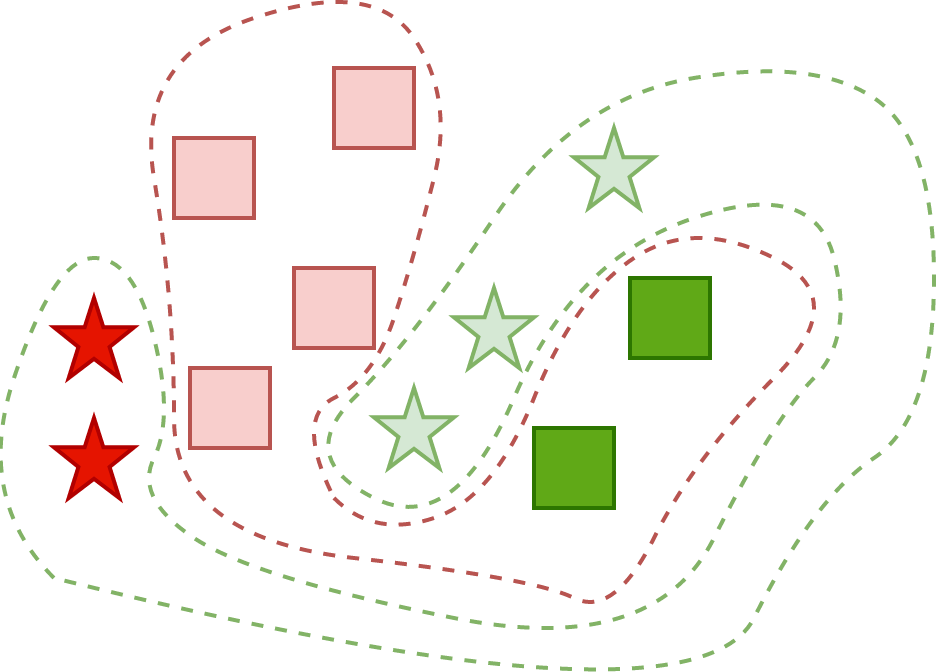}
    \caption{}
  \end{subfigure}
  \hfill
  \begin{subfigure}{0.22\linewidth}
    \label{fig:dbat_4_new}
    \includegraphics[width=0.97\linewidth]{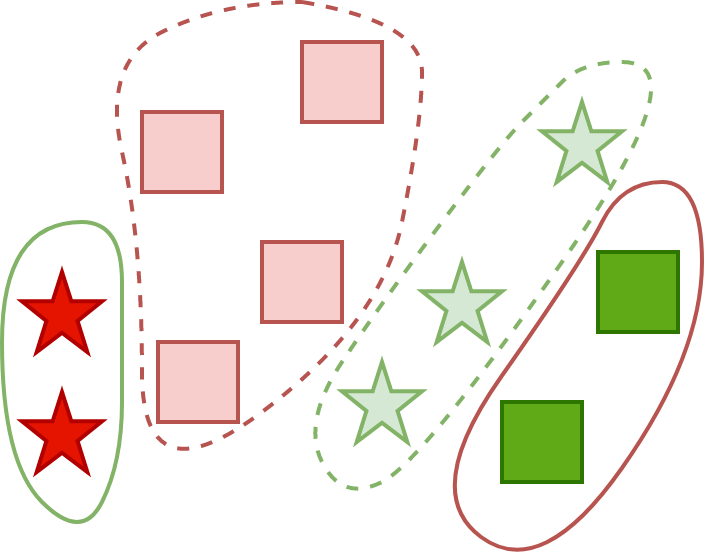}
    \caption{}
  \end{subfigure}
  \caption{Conceptual illustration of our method compared to standard adversarial training. (a) standard training decision boundary (b) generation of adversarial examples that cross the decision boundary (c) standard adversarial training that tries to learn boundaries for both clean and adversarial examples of each class (d) ``splitting the difference''
  (our method). The dashed lines represent the original classes' learned boundaries. The solid lines represent the adversarial classes learned by DBAT. Standard adversarial training learns complex shared boundaries for the two classes, while our method learns four much simpler boundaries.}
  \label{concept}
\end{figure*}

Usually, in standard adversarial training methods, each generated adversarial example is annotated with the source class label. Some works also attach a domain label ({\em clean} or {\em adversarial}) for enhanced techniques, such as using an additional domain classifier \cite{levi2021domain}, adversarial examples detection \cite{carlini2017adversarial, tramer2021detecting}, etc. On the contrary, we hypothesize that adversarial examples generated from a given source class induce a totally distinct class distribution. Therefore, in case one wishes to avoid significant natural accuracy degradation, adversarial training should be adjusted to take these additional classes into account during training. Overall, we make the following contributions:

\begin{itemize}
  \item We introduce a novel approach for training robust models,
  which departs from the established paradigm of attempting to learn a common
  decision boundary for each natural class and its adversarially perturbed version.
Rather, we claim that for each class, 
the adversarial perturbations induce a distinct distribution on the examples,
so much so that it makes more sense to learn it as a separate label, rather than attempting
to shoehorn it into the original one.
Thus, our method doubles the number of classes but ends up learning much simpler decision
boundaries; we provide both theory and experiments in evidence of the efficacy of this trade-off (more below).
To our knowledge, this approach of ``splitting the difference'' (which we formally dub {\em Double boundary adversarial training}, DBAT), is completely novel in the adversarial training setting.

\item 
We perform a comprehensive
battery of
experiments
to
demonstrate that our approach learns robust models 
while also achieving the highest reported natural accuracy, with a significant margin across different datasets. 
This optimal or near-optimal natural accuracy makes our method applicable for real-world applications (autonomous vehicles, face recognition systems, healthcare monitoring, and diagnosis, etc.) that cannot sacrifice natural accuracy for robustness. 
We stress that our aim is not to compete with the state-of-the-art models on
robustness, but rather a general-purpose technique for
endowing a classifier with a significant level of robustness,
while only incurring a minor degradation of natural accuracy. 
\footnote{Our source code is available on  \href{https://github.com/matanle51/Splitting-the-Difference-on-Adversarial-Training}{Github}.}

\item In \cref{sec:theory}, we 
provide a rudimentary plausibility argument
to shed some theoretical light
on
the statistical trade-off presented by DBAT: an increased number of classes to learn, but with much simpler boundaries.

\end{itemize}


\section{Related work}

Since the discovery of adversarial examples by \cite{szegedy2013intriguing}, a wide range of defenses were proposed to enhance robustness. Among these, adversarial training \cite{goodfellow2014explaining, madry2017towards} emerged as one of the most successful methods to train robust models. Madry \etal \cite{madry2017towards} proposed a technique, commonly referred to as standard Adversarial Training (AT), to minimize the cross entropy loss only on adversarial examples with respect to the original class labels.
Throughout the years, standard adversarial training was enhanced in various ways \cite{bai2021recent} -- with changes in the regularization terms \cite{kurakin2016atscale, zhang2019theoretically, wang2019improving, kannan2018adversarial, goldblum2020adversarially, lee2020adversarial, levi2021domain}, model ensemble  \cite{tramer2017ensemble, pang2019improving, yang2020dverge}, adversarial training with adaptive attack budget \cite{ding2018mma, cheng2020cat}, curriculum adversarial training \cite{cai2018curriculum, zhang2020attacks, wang2019convergence}, utilizing out-of-distribution data \cite{lee2021removing}, applying Stochastic Weight Averaging (SWA) \cite{izmailov2018averaging} to flatten the adversarial loss landscape \cite{gowal2020uncovering, chen2020robust}, 
adapting adversarial training to model weights using Adversarial Weight Perturbation (AWP) \cite{wu2020adversarial, tsai2021formalizing},
and combining adversarial training with data augmentation techniques \cite{gowal2021improving, rebuffi2021fixing, rebuffi2021data} and synthetically generated data \cite{sehwag2021robust, pang2022robustness, wang2023better, xu2023exploring}.

Other lines of research include theoretically certified approaches \cite{cohen2019certified, raghunathan2018certified, sinha2017certifiable, raghunathan2018semidefinite, wong2018scaling, wong2018provable, gowal2018effectiveness}, computationally efficient adversarial training \cite{shafahi2019adversarial, wong2020fast, andriushchenko2020understanding, zhang2019you, sriramanan2021towards}, robust overfitting and possible mitigations \cite{rice2020overfitting}, semi/un-supervised adversarial training \cite{carmon2019unlabeled, uesato2019labels, zhai2019adversarially}, 
adversarial self-training and pre-training \cite{jiang2020robust, chen2020adversarial}, incorporating domain adaptation alongside adversarial training \cite{song2018improving, levi2021domain}, and  robust model architecture and custom building blocks \cite{xie2019intriguing, xie2019feature, zhang2019defense, wang2022removing}. 
\mycolorgreen
Specifically, normalizer-free robust training (NoFrost) \cite{wang2022removing} suggested removing all batch normalization (BN) layers from the network during AT, but this approach was shown to have a negative effect on the robustness against stronger attacks. \footnote{See: \href{https://github.com/amazon-science/normalizer-free-robust-training/issues/2}{AutoAttack reduces accuracy of NoFrost}.}
\normalcolor

Some well-known methods include 
\cite{zhang2019theoretically}, who proposed the method TRADES, which uses the Kullback-Leibler (KL) divergence as a regularization term to push the decision boundary away from the data. 

\mycolorgreen
Most related methods to DBAT are ones that suggest changes to the regularization terms of adversarial training (AT) with the goal of reducing natural accuracy degradation in AT. A recent work by \cite{cui2021learnable}, named \textit{LBGAT}, 
used an additional Mean Square Error (MSE) regularization term between the logits of a natural model, alongside the robust. In \cite{cheng2020cat}, authors suggested a work named Customized Adversarial Training (\textit{CAT}) which adaptively customizes the perturbation level and the corresponding label for each training sample, but was later shown to suffer from obfuscated gradients \cite{sitawarin2021sat}. Another work \cite{rade2021reducing} suggested Helper Adversarial Training (\textit{HAT}), which attempts to mimic the discriminative features learned by standard trained networks to improve the accuracy of clean samples with the goal of improving natural accuracy. Recently, Universal Inverse Adversarial Training (\textit{UIAT}) was suggested by \cite{dong2023enemy} to encourage the model to produce similar output probabilities for an adversarial example and its ``inverse adversarial'' counterpart, where the counterpart is generated by maximizing the likelihood in the neighborhood of the natural example. Additionally, the authors of \cite{wang2023generalist} recently suggested \textit{Generalist}, which consists of two base learners separately trained within their respective fields and a
global learner that aggregates the parameters of base learners during the training process. The parameters of base learners are collected and combined to form a global learner at intervals during the
training process.
\normalcolor

In contrast to all of the aforementioned methods, our work suggests a fundamentally different approach.
While other methods treated the generated adversarial examples of a given class as additional instances of that same class when learning the class boundaries (or even used only the adversarial examples), 
our method acknowledges the fact that adversarial examples induce additional class distributions on the source dataset, which essentially doubles the number of classes in the dataset, and therefore these examples should be treated as additional classes of the dataset. We underline that our goal is not to improve robust accuracy compared to the current state-of-the-art in robustness, but rather to equip models with a significant level of robustness, while keeping their natural accuracy as high as possible.

\section{Double Boundary Adversarial Training}

In this section, we introduce our approach for training robust models, 
\textit{Double Boundary Adversarial Training}, DBAT. A conceptual illustration is presented in Figure \ref{concept}, \mycolorgreen and is supported empirically in Appendix \ref{distance-decision}, where we test if such a case is possible by calculating the distance histogram of random examples to the decision boundary.
\normalcolor

\subsection{Motivation behind Double Boundary Adversarial Training}
Tsipras \etal \cite{tsipras2018robustness} argued that robustness may be at odds with natural accuracy, and usually the trade-off is inherent. We 
concur
that this 
indeed
is
typically
the case when adversarial examples are assigned to the same class as the natural examples they were generated from. However, when separating the adversarial examples from their source, and generating new parallel adversarial classes, we may be able to maintain natural accuracy while still achieving significant robustness.

Therefore, we suggest an alternative approach to learning robust models. Our main hypothesis is that natural examples and their adversarial counterparts should not necessarily be assigned to the same class. Instead, for each class, we learn an additional counterpart adversarial class, which will be assigned to the adversarial examples. In essence, the number of classes in the dataset is doubled.

In other words, instead of learning shared boundaries for both natural examples and their adversarial counterparts, where adversarial examples are expected to reside in the same class as their natural counterparts, we suggest treating adversarial examples as additional classes in the dataset. 

While other methods can be thought of as modifying existing boundaries, DBAT learns boundaries for completely new dynamically generated adversarial classes.
We hypothesize that this behavior creates a trade-off, where on one hand, DBAT does not induce significant changes 
to existing boundaries for natural classes in terms of complexity, and keeps them smoother (as we empirically demonstrate on the synthetic dataset experiment in \cref{synthetic} and Figure \ref{toy-viz}) - which mitigates the drop in natural accuracy. But on the other hand - it's a more challenging task to learn completely new classes, which in turn can impact robustness in some tasks.

We highlight that our aim is not to compete with the state-of-the-art on robustness, but rather to find a general-purpose technique that equips classifiers with a significant level of robustness with only minor or neglectable degradation of their natural accuracy.

\subsection{Training procedure}
During the training process, our goal is to learn additional classes, one for each in the original class set.
Given a dataset $\mathcal{S}={\{(x_i,y_i)\}}_{i=1}^n$ with $\mathcal{C}$ classes $\mathcal{Y}=\{0,1,...,\mathcal{C}-1\}$, we define a new class space $\mathcal{Y}_{\small{DBAT}}=\{0,1,...,\mathcal{C}-1, \mathcal{C},\mathcal{C}+1, ..., \mathcal{C} \cdot 2-1\}$ where class label $\mathcal{C}+k,  k\geq0$ is the label of the adversarial examples generated for class $k$.

For each natural example $(x_i,y_i)$, DBAT generates an adversarial example $x_i'$ using targeted-PGD with a random target.
\textit{Then, the adversarial example is assigned with the adversarial class $y_i+\mathcal{C}$ which corresponds to the natural class $y_i$ of the natural example}. To summarize, for each natural example $(x_i,y_i)$, we generate adversarial example and assign it the corresponding adversarial class, $(x_i', y_{i}+\mathcal{C})$. Algorithm \ref{train-algo} 
describes the training procedure.

\noindent
{\bf Remark.}
We 
note that
targeted-PGD typically does not outperform untargeted-PGD when used with standard adversarial training methods. That said, 
we argue
that this observation is not applicable to our setting.
First, since we aim to learn {\em new} generic adversarial classes, it stands to reason that class diversity will be conducive to generalization.
Therefore, using random targeted-PGD mitigates the scenario where adversarial examples generated by untargeted attacks
for a given class focus on small/specific regions of the manifold.
Additionally, using untargeted PGD, attacks can potentially be directed to the adversarial class corresponding to the natural one. To avoid the later, one can use untargeted PGD only on the original classes, by adding a projection back to the original classes during the optimization.
Finally, we also experimented with using Least-Likely targeted-PGD. See Appendix \ref{targeted-pgd} for results comparison.

\begin{algorithm}[tb]
   \caption{DBAT Training}
   \label{train-algo}
\begin{algorithmic}
   \STATE {\bfseries Input:} $\mathcal{S}={\{(x_i,y_i)\}}_{i=1}^n$ with $\mathcal{C}$ classes, and model $f_{\theta}$
   \STATE {\bfseries Parameters:} Batch size $m$, perturbation size $\epsilon$, attack step size $\tau$, current iteration index $k$ (zero-initialized), and learning rate $\alpha$ 
   \REPEAT 
   \STATE Fetch mini-batch $X_{s}={\{x_j\}}_{j=1}^m$, $Y_{s}={\{y_j\}}_{j=1}^m$
   \STATE Initialize $X^{'}=\{\}, Y^{'}=\{\}$
   \FOR{$j=1$ {\bfseries to} $m$ (in parallel)}
    \STATE   { 

    $\textit{\# Generate an adv. example}$

     \STATE { $y'_{j} = \text{Select random label uniformly from}
  \enspace \{0,1,...,\mathcal{C}-1, \mathcal{C},..., \mathcal{C} \cdot 2-1\}/\{j, j+\mathcal{C}\}$}
    
    $x'_{j} = \text{targeted-PGD}(x_j,y'_{j},\epsilon,\tau, f_{\theta})$

    $\textit{\# Save the adv. example with the adv. class label}$
    
    $X^{'} = X^{'} \cup \set{x'_j}$
    
    $Y^{'} = Y^{'} \cup \set{y_j + \mathcal{C}}$
  }
  \ENDFOR
  
  \STATE $\theta = \theta - \alpha\cdot\nabla_{\theta}\ell(X_{s}\cup X^{'}, Y_{s}\cup Y^{'})$
  
\STATE $\theta' = \dfrac{\theta'\cdot k + \theta}{k+1}$
\STATE $k = k + 1$
  \UNTIL{stopping criterion is met}
\end{algorithmic}
\end{algorithm}

\subsection{Inference procedure}
\label{inference}
At inference time, the model will output a probability vector $v$ of size $|v| = 2\cdot \mathcal{C}$ which corresponds to the double number of classes used during training. However, the dataset originally has only $\mathcal{C}$ classes.
Therefore, as our final class prediction, we use the following formula:
\begin{equation} 
    v^{*} = (\max (v_0,v_{\mathcal{C}}), ..., \max(v_{\mathcal{C}-1},v_{2\cdot\mathcal{C}-1})),
\end{equation} 
\begin{equation} 
    \text{predicted class} = \argmax_{0 \leq i \leq \mathcal{C}}{v_{i}^{*}}.
\end{equation} 

In other words, the final class prediction is taken as the class with the maximum probability. If this class is one of the adversarial classes, we return to its natural counterpart.

\subsection{Illustrating DBAT's Decision Boundaries using a Synthetic Dataset}
\label{synthetic}
In Figure \ref{toy-viz} we 
illustrate how DBAT can learn simpler and smoother decision boundaries by 
applying
it
to
a synthetic dataset. We 
exhibit
the decision boundaries for standard AT and DBAT. 
The dataset is composed of isotropic Gaussian blobs with a cluster standard deviation of 0.1 and two features generated using the \verb+make_blobs+ from \cite{scikit-learn}.  
The number of samples for each blob is 10,000. 
The adversary was given a budget of $\epsilon=1.2$ optimized for six steps with a step size of $0.2$. This enables some of the samples to cross the decision boundary. As can be seen in Figure \ref{fig:dbat}, our method learns much smoother and simpler decision boundaries as compared to standard adversarial training in Figure \ref{fig:at}. 

\begin{figure*}[ht]
  \centering
  \begin{subfigure}{0.33\linewidth}
    \centering
    \includegraphics[width=0.8\linewidth]{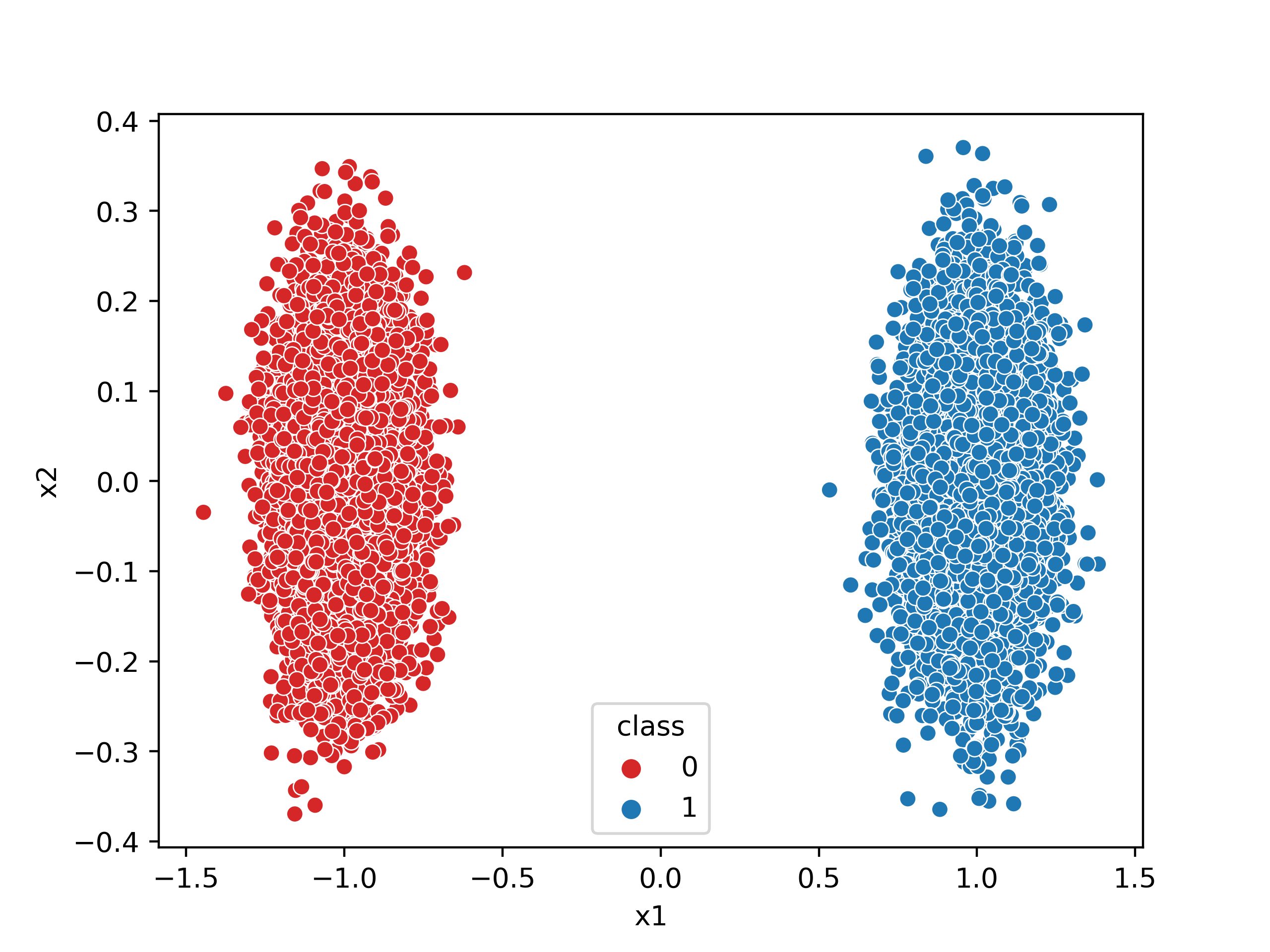}
    \vskip -0.05in
    \caption{Isotropic Gauss. blobs (boundary $x_1=0$)}
    \label{fig:ds}
  \end{subfigure}
  \hfill
  \begin{subfigure}{0.33\linewidth}
    \centering
    \includegraphics[width=0.8\linewidth]{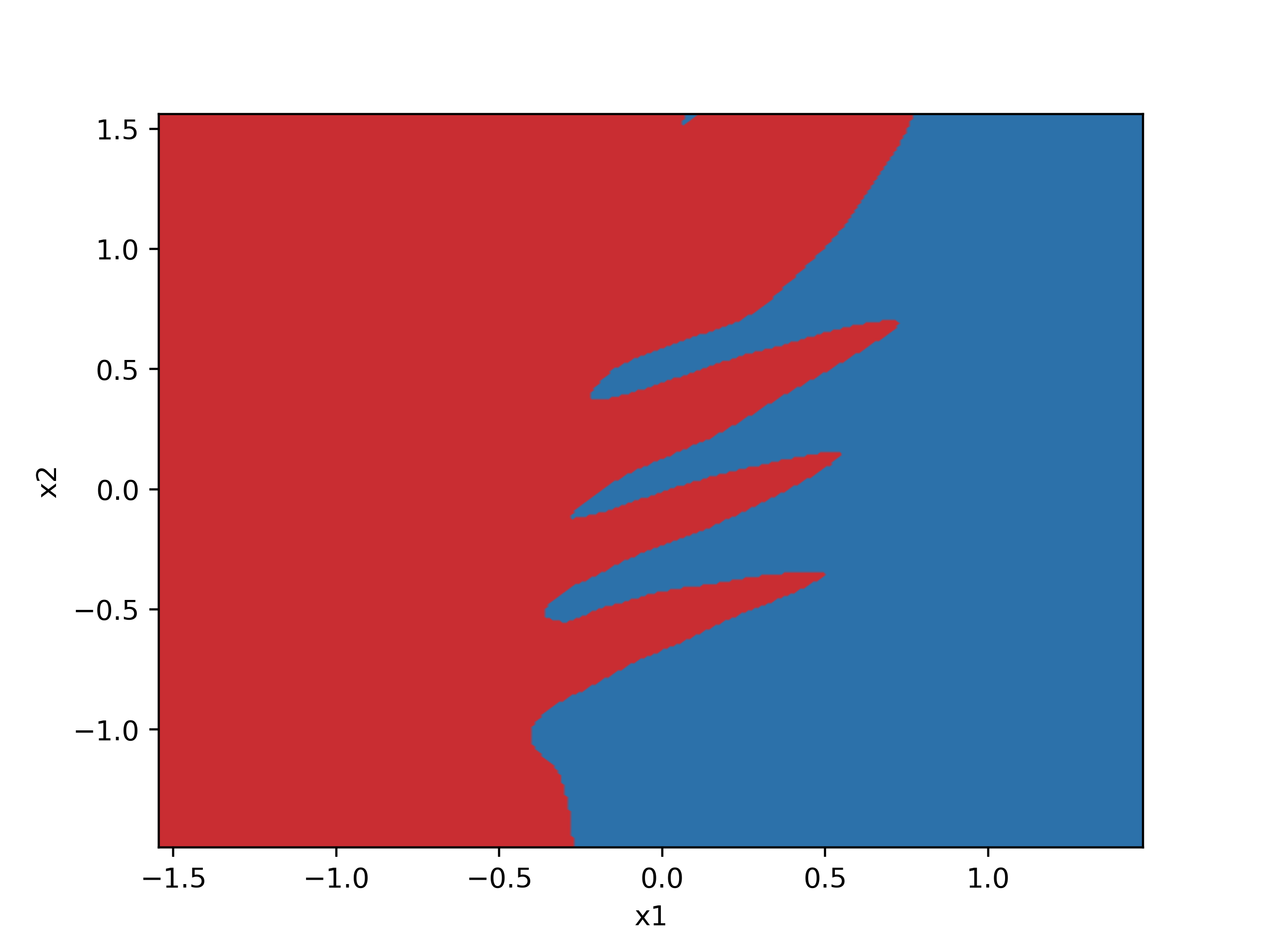}
    \vskip -0.05in
    \caption{Standard AT decision boundary}
    \label{fig:at}
  \end{subfigure}
  \hfill
  \begin{subfigure}{0.33\linewidth}
    \centering
    \includegraphics[width=0.8\linewidth]{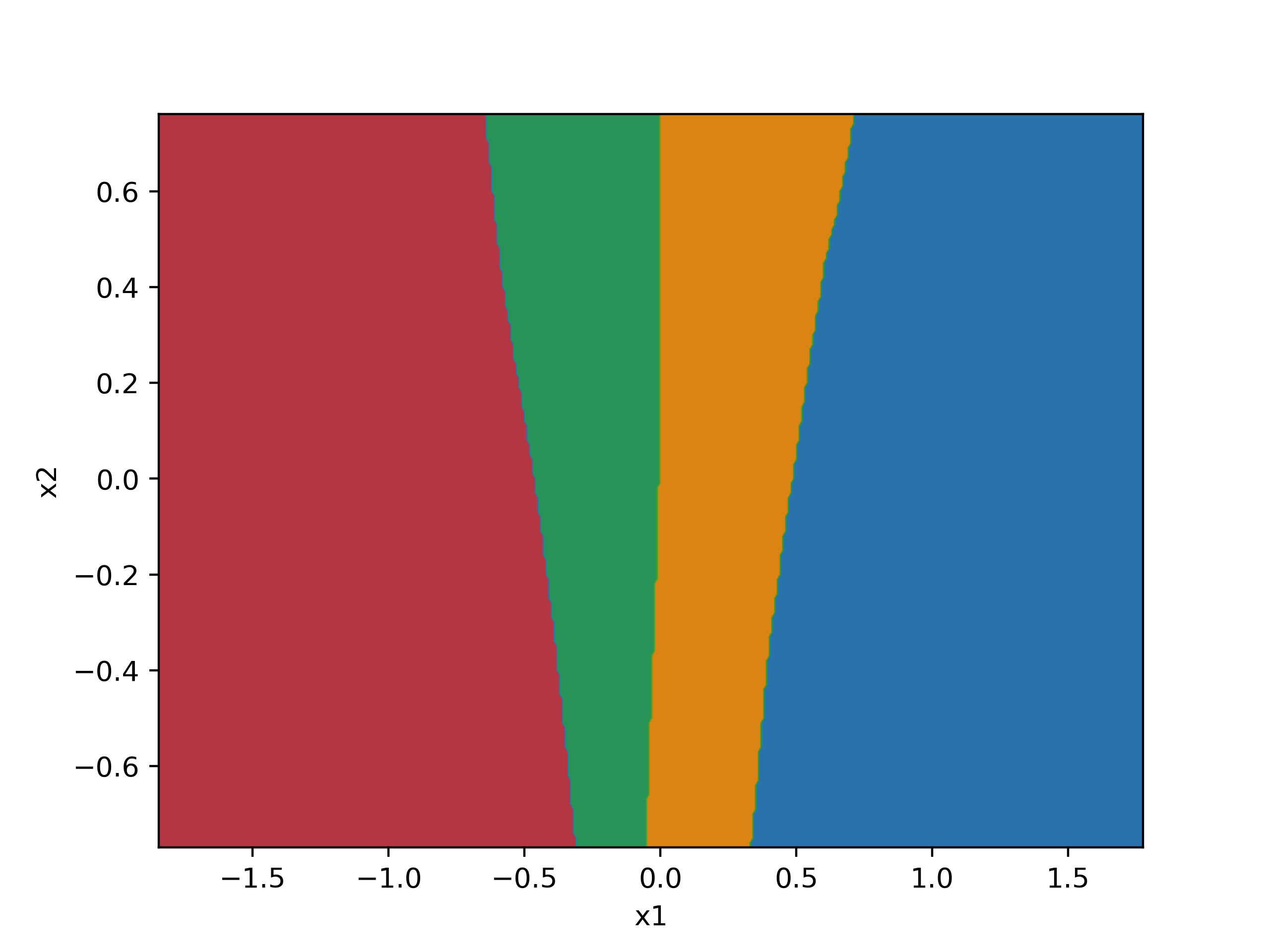}
    \vskip -0.05in
    \caption{\textbf{DBAT} decision boundaries}
    \label{fig:dbat}

  \end{subfigure}
  \caption{Synthetic dataset viz. on 2-classes dataset (a) of two 2D features each. Adversary: 6-step $\ell_\infty$-PGD, $\epsilon=1.2$, $\delta=0.2$.}
  \label{toy-viz}
\end{figure*}


\begin{figure*}[ht]
\mycolor
  \centering
  \begin{subfigure}{0.3\textwidth}
     \centering
    \includegraphics[width=\linewidth]{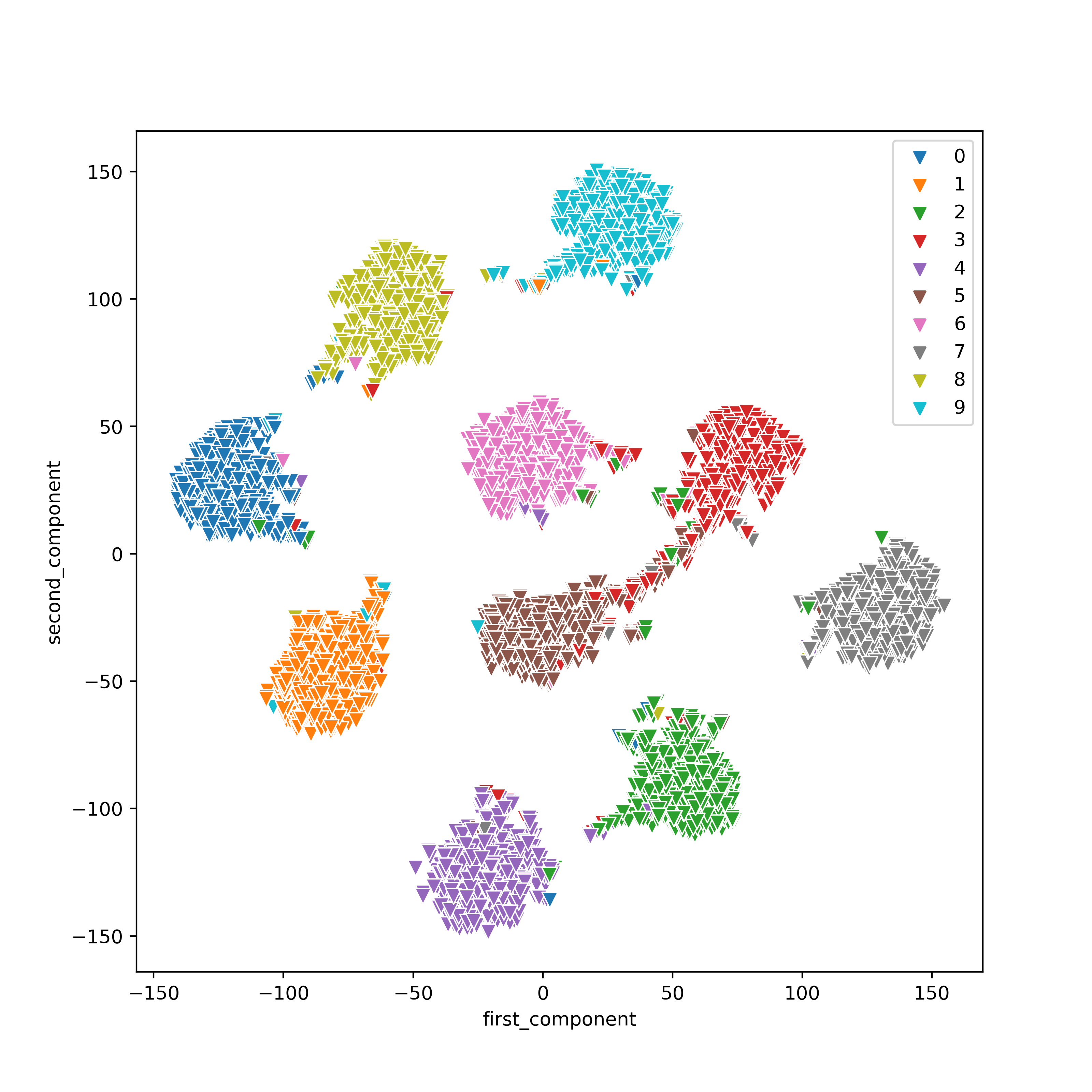}
    \caption{\mycolor DBAT logits for natural examples and original classes}
  \end{subfigure}
  \hspace{0.07\textwidth}
  \begin{subfigure}{0.3\textwidth}
     \centering
    \includegraphics[width=\linewidth]{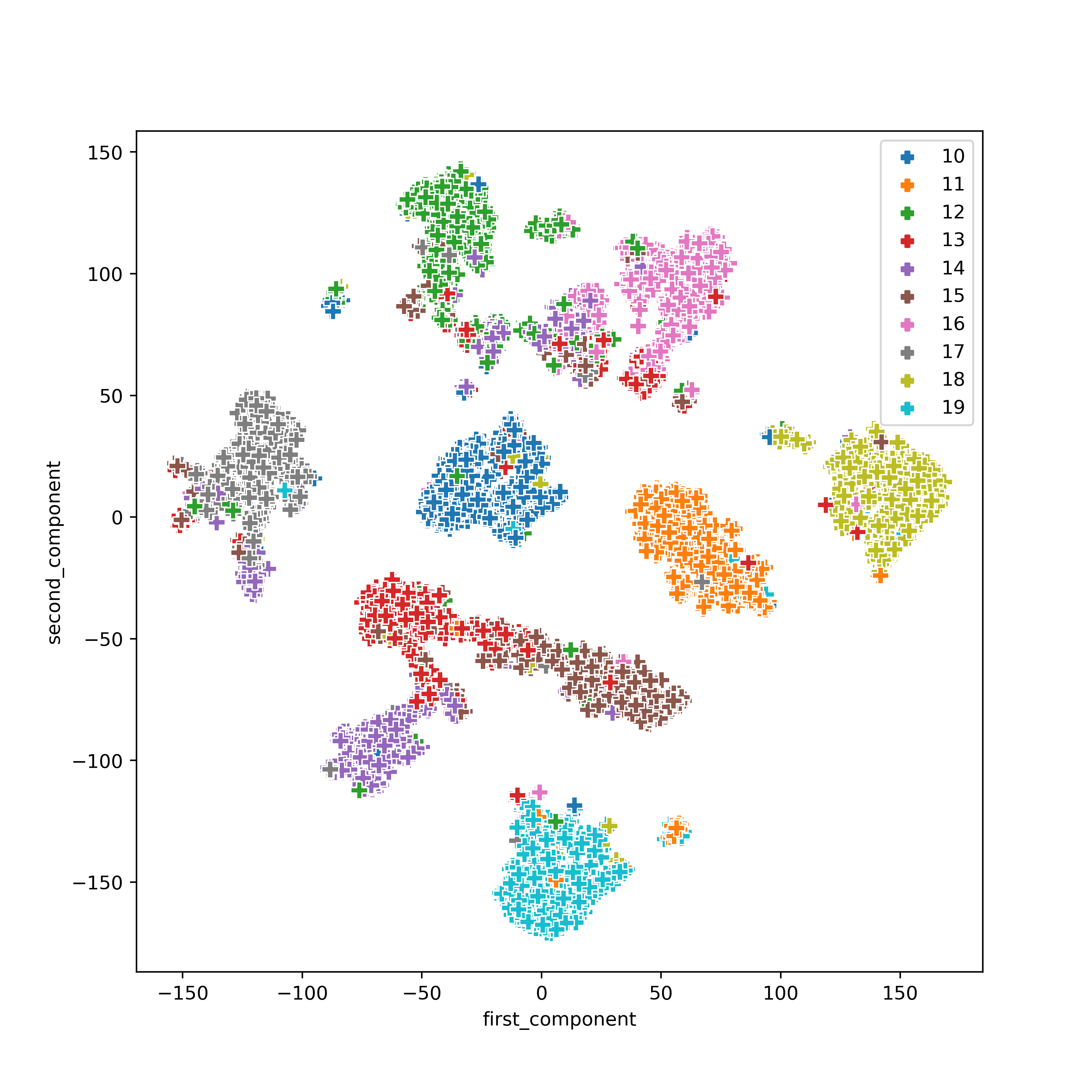}
    \caption{\mycolor DBAT logits for adv. examples on newly generated adv. classes}
  \end{subfigure}
  \hspace{0.07\textwidth} \\
  \begin{subfigure}{0.3\textwidth}
   \centering
    \includegraphics[width=\linewidth]{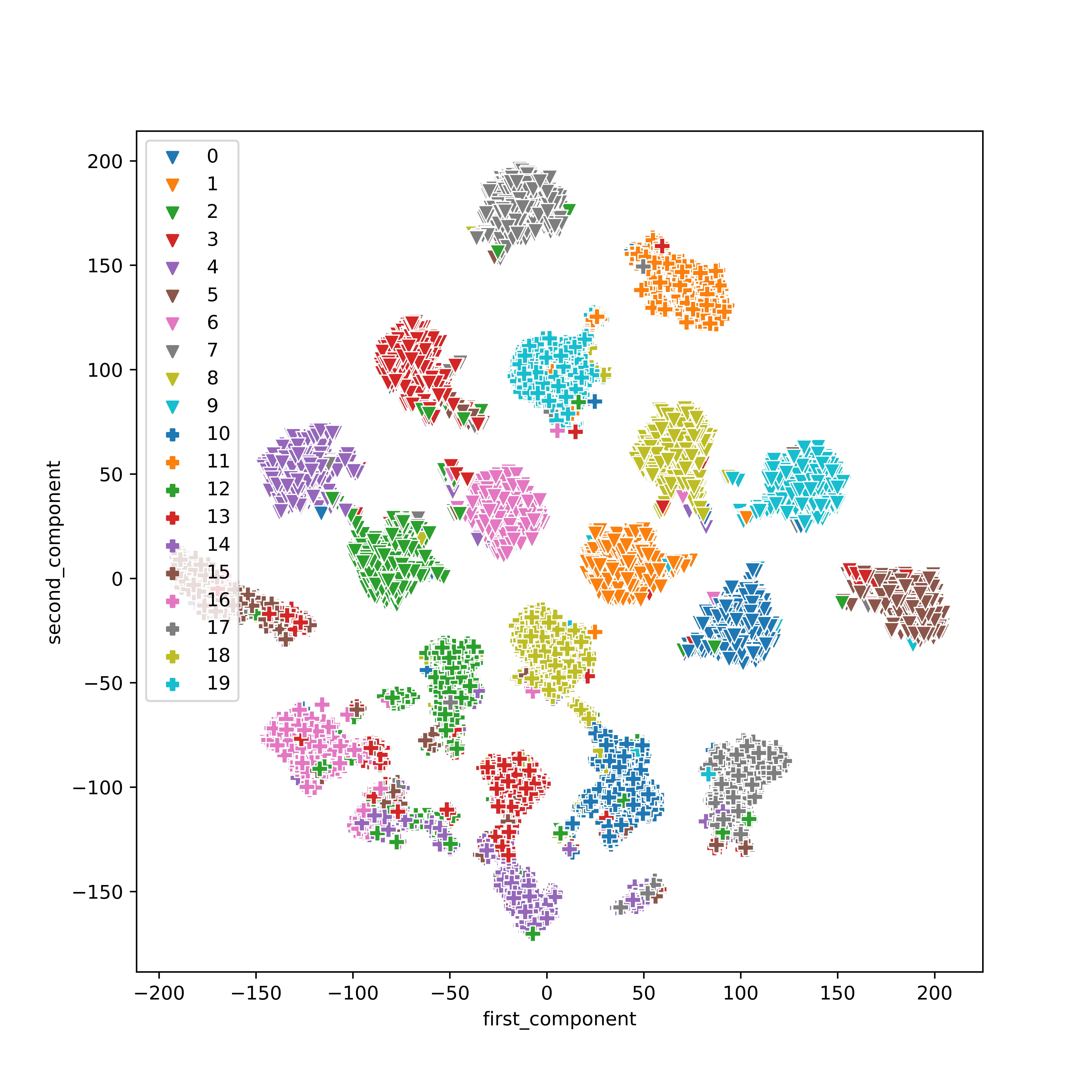}
    \caption{\mycolor DBAT logits for both natural and adv. examples on all classes}
  \end{subfigure}
  \hspace{0.07\textwidth}
  \begin{subfigure}{0.3\textwidth}
   \centering
    \includegraphics[width=\linewidth]{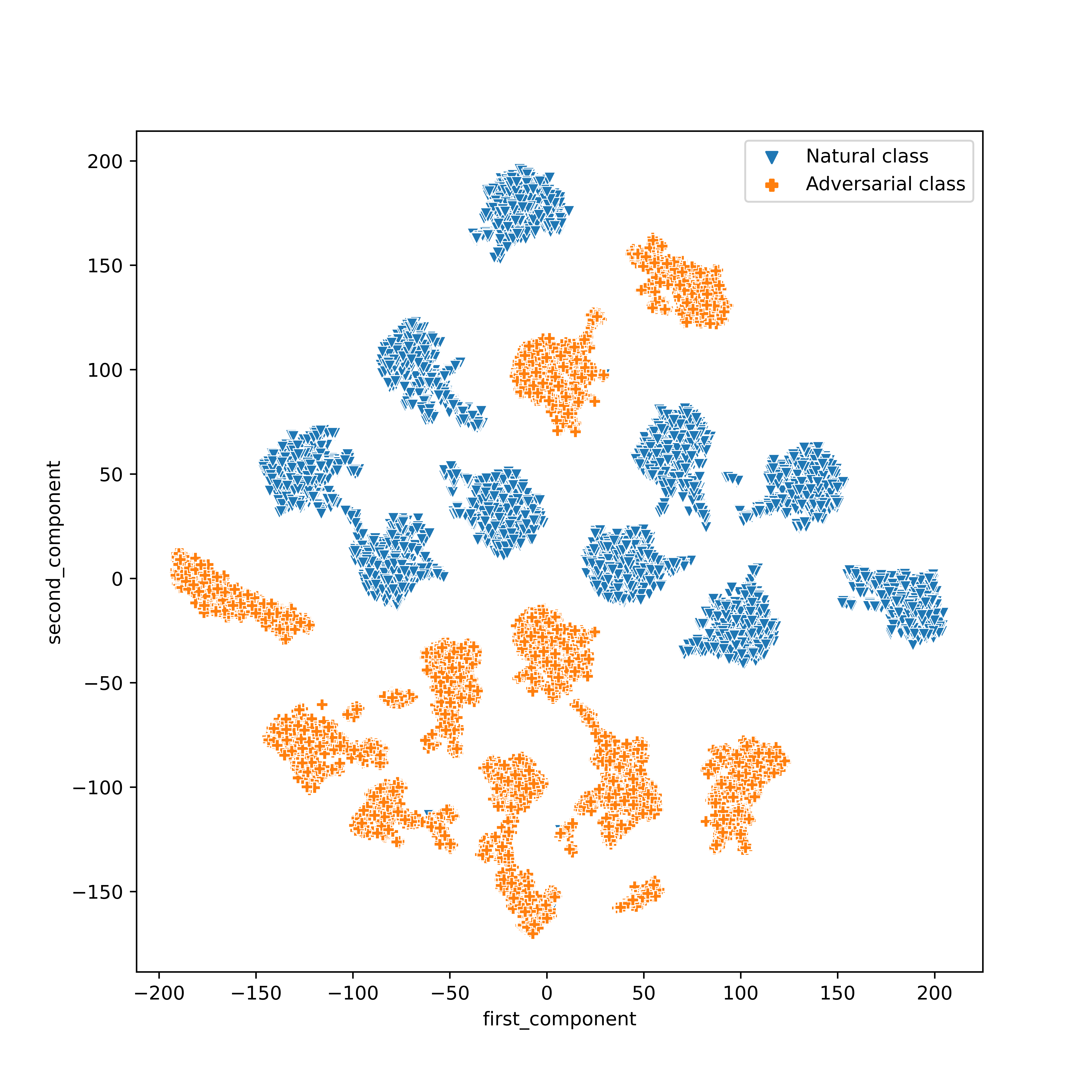}
    \caption{\mycolor DBAT logits in two colors for natural (blue) and adv. examples (orange).}
  \end{subfigure}
  \caption{\mycolor Visualizing the logits of DBAT on CIFAR-10 test set using T-SNE \cite{van2008visualizing} with two components on the model output for (a) natural examples (b) adversarial examples with their new generated adversarial classes (c) combined 2-D visualization of both natural and adversarial examples with all 20 classes (same color for natural class and its adversarial counterpart). (d) same plot as in c, but colored in two colors to observe the separation between natural and adversarial examples.  We can observe the strong separation between classes obtained by DBAT, for both original and newly generated classes. Interestingly, adversarial and natural examples almost don't mix, and the majority of mismatches within each cluster are from the same domain (adversarial/natural).}
  \label{tsne}
\end{figure*}

\mycolor

\section{Theoretical analysis}
\label{sec:theory}

We provide a theoretical plausibility argument for the empirical success of our approach, in the following somewhat idealized setting. 
We identify a phenomenon, which we term the {\em DBAT advantage}, which, when applicable, justifies the use of our technique.

Here we assume familiarity with
the basic notions of PAC learning, such as
the {\em sample error} of a hypothesis,
$\serr(h)$, its {\em generalization error},
$\gerr(h)$, and the Vapnik-Chervonenkis (VC) dimension of a concept class;
these may all be found, e.g., in
\cite{MR1741038}.
Suppose that one trains a $k$-multiclass classifier by reducing it to $k$ binary classification problems via the standard $1$-vs-all method
(i.e., a separate in-class/out-class binary classifier is trained for each of the $k$ classes). 
Suppose for simplicity that
each of the classifiers is trained using the
same concept class $H$ of VC-dimension $V$.
If $h_i$ is the classifier trained for the $i$th class on a sample of size $n$
with sample error $\serr(h_i)$, then
the agnostic PAC bound \cite[Theorem 4.9]{MR1741038}
implies that with probability at least $1-\delta$,
\beqn
\textstyle
\label{eq:vc-bd}
\gerr(h_i)-\serr(h_i)
&\le&
c\paren{
\sqrt{\paren{V+\log(1/\delta)}/n}
},
\eeqn
where 
$\gerr(h_i)$ is the 
generalization error
and $c>0$ is a universal constant.
\paragraph{Claim.}
The following form of
\eqref{eq:vc-bd} to holds for all of the $k$ classifiers,
with probability at least $1-\delta$,
simultaneously:
\begin{equation}
\begin{aligned}
\textstyle
\label{eq:vc-bd-union}
\max_{1\le i\le k}
\gerr(h_i)-\serr(h_i)
\le
c\paren{
\sqrt{\paren{V+\log(k/\delta)}/n}
}
\end{aligned}
\end{equation}
\paragraph{Proof.}
One sets $\delta'=\delta/k$, which guarantees, with probability at least
$1-\delta'$, a generalization error of at most
$c\sqrt{(V+\log(1/\delta'))/n}$
for each class individually, and hence, by a union bound,
for all classes simultaneously, with probability of at least
$1-k\delta'=1-\delta$.
$\blacksquare$

But now suppose that we can express
each $h\in H$ as a union of two simpler concepts:
$h=h_1\cup h_2$, where $h_1,h_2\in H'$, and the latter
has VC-dimension, say, $V/2$.
In this case, we can 
formulate the learning problem
as a $2k$-multiclass classification problem, over the concept class $H'$. By assigning $V/2$ and $2k$, the corresponding bound
in \eqref{eq:vc-bd-union} will now behave as:

\begin{equation}
    \small
    \sqrt{({V/2+\log(2k/\delta)}){n}}
\end{equation}

--- which, for constant $\delta$ and large $V$,
constitutes considerable savings in sample complexity.
The improvement in sample complexity will be even more significant as we consider
$\ell$-fold (rather than just $2$-fold) unions of basic concepts:
$h=h_1\cup h_2\cup\ldots\cup h_\ell$, $h_i\in H'$.
We will refer to this phenomenon 
--- in which decreasing hypothesis complexity while increasing
the number of classes reduces the overall sample complexity ---
as the {\bf DBAT advantage}, 
and discuss it in greater detail
below.

\theoremstyle{plain}
 \newtheorem{claim}{Claim}

We will illustrate this phenomenon in some detail
on the natural example of halfspaces and Euclidean balls in $\R^d$. Since providing the requisite background on Vapnik-Chervonenkis (VC) theory
(in particular: shattering, VC-dimension)
is beyond the scope of the paper,
we refer the reader to 
\cite{mohri2018foundations}.

\paragraph{Halfspaces.}

\begin{claim}
If $H'$ is the collection of all 
homogeneous (going through the origin)
halfspaces 
$\R^d$,
then the VC-dimension 
of $H'$ is $d+1$.
\end{claim}
\begin{proof}
It is shown in \cite[Example 3.2]{mohri2018foundations}
that the VC-dimension of {\em general}
halfspaces in $\R^d$ is $d+1$. 
The restriction that the halfspace
contain the origin can be ensured by
translating any shattered set
$\{x_1,\ldots,x_{d+1}\}$
by $x_1$
to obtain the shattered set
$\{x_2-x_1,\ldots,x_{d+1}-x_1\}$
of size $d$. This shows that homogeneous
hyperplanes have VC-dimension $1$ less than
the general ones, i.e., $d$.
\end{proof}

\begin{claim}
\label{union-halfspaces}
For $H'$ as above
(the collection of all homogeneous halfspaces 
in
$\R^d$),
the set of all $2$-fold unions of concepts from $H'$
will have VC-dimension at least twice
that of $H$. 
\end{claim}

\begin{proof}
For the lower bound,
it suffices
to find $d$ points in the positive orthant
shattered by a set of $2^{d}$ homogeneous
halfspaces $H_1\subset H'$, 
as well as another set of 
$d$ points in the negative orthant
shattered by another set of $2^{d}$ 
homogeneous
halfspaces $H_2\subset H'$, 
such that each $h\in H_1$ labels the negative orthant negative,
while each $h\in H_2$ labels the positive orthant positive.
Evidently, the set of pairwise unions of $h_1\in H_1$
and $h_2\in H_2$ shatters the combined set of $2d$ points.
\end{proof}

This example illustrates that $2$-fold unions
of simple classifiers can double the VC
dimension of the hypothesis class.
In the more general case of $\ell$-fold unions of hyperplanes,
it is known
\cite{DBLP:journals/jmlr/CsikosMK19}
that the VC-dimension is $\Omega(\ell d\log \ell)$,
so the increase in sample complexity is even more significant.
Moreover, 
\cite{DBLP:journals/jmlr/CsikosMK19} showed that
this continues to be true
for many other kinds of Boolean aggregations:
intersections, XORs etc.

\newpage
\paragraph{Euclidean balls.}

\begin{claim}
\label{vc-dim-balls}
The VC-dimension of Euclidean balls in
$
\R^d
$
is $d+1$.
\end{claim}
\begin{proof}
This is a well-known fact,
which we prove for completeness
and also because the argument will be useful
in the sequel. 

Both the upper and lower bounds
on the VC-dimension of balls rely on the fact
that locally, these act like halfspaces:
any two finite sets separated by a halfspace
can also be separated by a ball of large enough
radius (see Figure \ref{theory_ball_example}).

This argument is enough to establish the lower bound: any set that is shattered by general halfspaces is also shattered by Euclidean balls (see Figure \ref{theory2_balls}),
and we know from 
\cite[Example 3.2]{mohri2018foundations}
that such a set can be as large as $d+1$.

For the upper bound, we invoke
Radon's theorem 
\cite[Theorem 3.4]{mohri2018foundations}:
Any set $S$ of $d+2$ points in $\R^d$ can be partitioned into two subsets $S_1$ and $S_2$ such
that the convex hulls of $S_1$ and $S_2$ intersect.
Such a partition will be called a {\em Radon partition}.
Suppose, for a contradiction, that
the Euclidean balls shatter some set $S$
of $d+2$ points. Then there exists a
Radon partition of these into $S_1$ and $S_2$.
But shattering means that some ball $B_1$ contains
$S_1$ and not $S_2$, while another ball $B_2$
contains $S_2$ but not $S_1$.

This means that $S_1$ and $S_2$ must be separable by a hyperplane.
We conclude that the Euclidean balls cannot shatter
any more points that the halfspaces, which is 
at most $d+1$.
\end{proof}

\begin{figure}[h]
\centering
\begin{minipage}[b]{.45\columnwidth}
\centerline{\includegraphics[width=0.45\columnwidth]{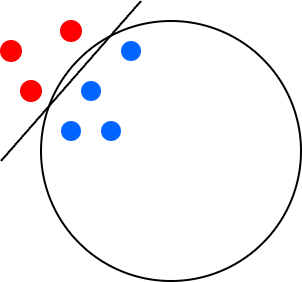}}
\caption{
\mycolorgreen Two finite sets separated by a halfspace can also be separated
by a ball of large enough radius.}
\label{theory_ball_example}
\end{minipage}\qquad
\begin{minipage}[b]{.45\columnwidth}
\centerline{\includegraphics[width=0.6\columnwidth]{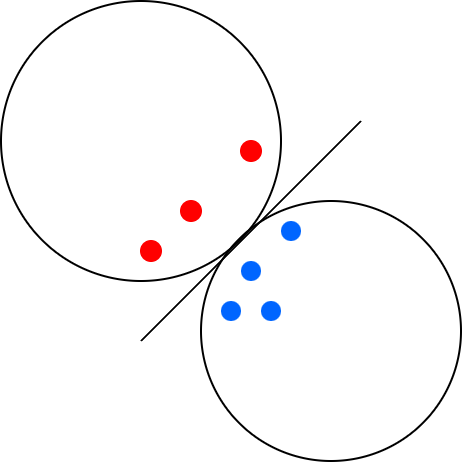}}
\caption{\mycolorgreen
A finite set that is shattered by balls is also shattered by halfspaces.}
\label{theory2_balls}
\end{minipage}
\end{figure}

\begin{claim}
If $H'$ 
is the collection of the Euclidean balls in $\R^d$
(with VC-dimension $d+1$, as shown above),
and
$H$
is the set of all $2$-fold unions of concepts from $H'$,
then the VC-dimension of $H$
is at least $2d$
\end{claim}
\begin{proof}
The argument proceeds by the same
reduction from balls to halfspaces
employed in the 
proof of the lower bound
in
Claim~\ref{vc-dim-balls}:
any set that can be shattered by halfspaces
can also be shattered by balls.
Now,
as in the proof of
Claim~\ref{union-halfspaces},
we construct two disjoint sets 
shattered by
homogeneous halfspaces,
consisting of $d$ points each,
in the positive and negative orthants, respectively.
Each is also shattered by balls,
and by dilating the points sufficiently
far from the origin, we can ensure that
the $2^d$ balls shattering the positive-orthant
set are disjoint from their counterparts shattering
the negative-orthant set.
Thus, the $2$-fold unions of Euclidean balls
shatter a set of size $2d$.
\end{proof}

The above discussion was more of a proof-of-concept illustration, since VC-dimension is not a particularly practical tool in analyzing deep neural networks with a large number of weights. In Appendix \ref{sec:rade}, we show that the thrust of our point continues to hold for the Rademacher complexity as well, which is far more practical as far as providing finite-sample generalization bounds \cite{BartlettFT17,yin2019rademacher}. Using the analysis of \cite{foster-rakhlin19}, 
we show that the Rademacher complexity of $\ell$-fold unions grows with $\ell$ roughly as $\sqrt \ell$.

An additional qualification of our
plausibility argument is 
that adversarial loss is distinct from the 0-1 loss discussed above.
This is indeed a limitation of our analysis, although in some instances it is possible to control
adversarial risk via a VC-type analysis
\cite[Theorem 2]{attias2019improved}.
Finally, an implicit assumption we have made above is that the adversarial perturbations are {\em non-adaptive}: the adversary has fixed a (possibly, stochastic) perturbation function in advance of seeing any data 
---
e.g., a neural network trained on a hold-out set, similar to black-box settings. This lets us argue that the examples continue to be iid, under a new (unknown, perturbed) distribution. This assumption, while
not entirely realistic, is often made to facilitate analysis \cite{attias2019improved}.

Modulo these qualifications, the above discussion provides evidence
that
when training $k$ classifiers
from a concept class with 
high complexity,
it may be advantageous to decompose them into unions
of simpler classifiers.
The blow-up of the number of classifiers is more than
compensated in the reduction of classifier complexity.

\normalcolor

\section{Experiments} \label{experiments}


To emphasize the advantage of Double Boundary Adversarial Training, we conduct extensive evaluations. The evaluation process of DBAT includes white-box and black-box settings, Auto-Attack, natural corruptions \cite{hendrycks2018benchmarking}, unforeseen adversaries, and ablation studies. All results are averaged over 5 runs while omitting one standard deviation. These evaluations demonstrate that the results obtained
are not a consequence of what is commonalty referred to as \textit{obfuscated gradients}~\cite{athalye2018obfuscated}.

We compare our method to some of the most well-known adversarial training methods -- Standard AT \cite{madry2017towards}, and TRADES \cite{zhang2019theoretically}, alongside related work -- LBGAT \cite{tsai2021formalizing}, Generalist \cite{wang2023generalist}, CAT \cite{cheng2020cat}, HAT \cite{rade2021reducing}, and UIAT \cite{dong2023enemy}.
Our evaluation starts with the common CIFAR-10 benchmark. In \cref{corruptions} and \cref{generalization}, we demonstrate the generalization of our method to other datasets by experiments on CIFAR-100 \cite{krizhevsky2009learning} and SVHN \cite{netzer2011reading}.
We use the WRN-34-10 \cite{zagoruyko2016wide} architecture for CIFAR-10 and CIFAR-100, and the PreAct ResNet-18 for SVHN. 
As suggested in \cite{rebuffi2021data}, we combine Stochastic Weight Averaging (SWA) \cite{izmailov2018averaging}, and Cutout \cite{devries2017improved}  with window length eight. We used "concatenated batches" as suggested by \cite{sitawarin2021improving}.  Attacks are generated using $\ell_{\infty}$-PGD with $\epsilon=8/255$, and perturbation step size 1/255 for 10 attack steps. Full experiment settings are detailed in Appendix \ref{exp-details}.

\subsection{Threat model}
\label{threat-model}
Our trained model
outputs a vector whose dimension is
twice
the number of classes in the dataset. 
That is, one-half of the 
coordinates
corresponds to the original classes, while the second half corresponds to the new adversarial classes.
Therefore, when considering the adversary's capabilities, specifically for untargeted white-box attacks, we need to explicitly define how the optimization is done.
Recall that the aggregation/projection described in \cref{inference} takes place only at inference time. Moreover, since the projection is not part of the computation graph, the defender can switch it to any desired metric (max, mean, median, log, exp, etc.) at any time during inference, without updating the network.
Therefore, there are three possible adversaries: 

First, the most basic adversary is one who does not utilize any projection function while attacking. 

Second, a more advanced (and perhaps most realistic) adversary, 
is
one who knows that the defender is utilizing a projection function, but does not have inference time access to the defender 
(e.g., a model was published at model zoos), and needs to conjecture the projection function while attacking.

Third, and most powerful adversary (although the least realistic one), is one who can 
access not only the entire network parameters but also real-time access to the defender's system and projection function at any given time during inference.

\begin{table}[h]
\mycolor

\caption{Natural, PGD
, and Auto-Attack (AA) robust accuracy against the fully adaptive white-box perfect knowledge adversary (i.e., "Inference real-time access"). The attack is an $\ell_{\infty}$ attack on CIFAR-10 with WRN-34-10. Similar to our method, the presented results are for models that \textbf{do not utilize additional data} during training. The Natural method refers to a model trained using standard training under the same hyper-parameters settings.
Green represents the best natural accuracy among the robust models. Orange represents the second-best natural accuracy among the robust models. The range alongside the green arrow represents the natural accuracy improvement compared to the other methods.}
\label{nat_aa}
\begin{center}
\begin{sc}
\begin{tabular}{lcccr}
\toprule
Method & Natural Acc. & PGD & AA \\
\midrule
\underline{DBAT} (Ours) & {\color{Green} \textbf{95.01} (\greenup 4--10.1\%)} & 54.61 &
40.08 \\
AT & 85.10 & 54.46 & 51.52 \\
TRADES & 84.92 & 55.56 & 53.08 \\
LBGAT  & 88.22 & 54.31 & 52.86 \\
Generalist & {\color{YellowOrange} 91.03} & 56.92 & 52.91 \\
HAT &  84.86 & 52.30 & 48.85 \\
UIAT & 85.01 & 54.63 & 49.11 \\
CAT & 89.61 & 73.38 & 34.78 \\
\midrule
Natural & 95.43 & 0 & 0 \\
\bottomrule
\end{tabular}
\end{sc}
\end{center}
\normalcolor
\end{table}

\begin{figure*}[ht]
  \centering
  \begin{subfigure}{0.62\columnwidth}
    \centering
    \includegraphics[width=\linewidth]{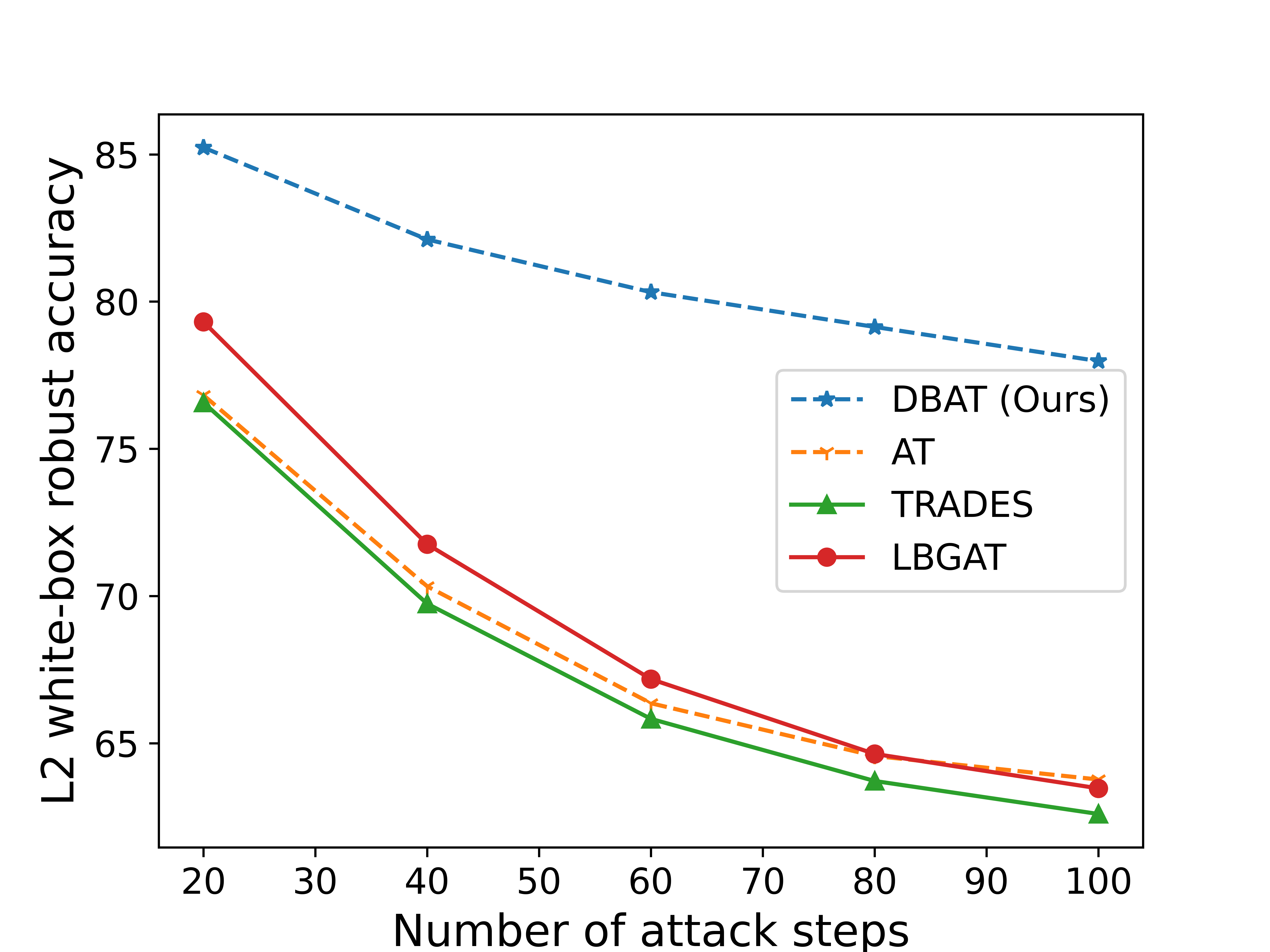}
    \caption{$\ell_{2}$-PGD}
  \end{subfigure}
  \begin{subfigure}{0.62\columnwidth}
    \centering
    \includegraphics[width=\linewidth]{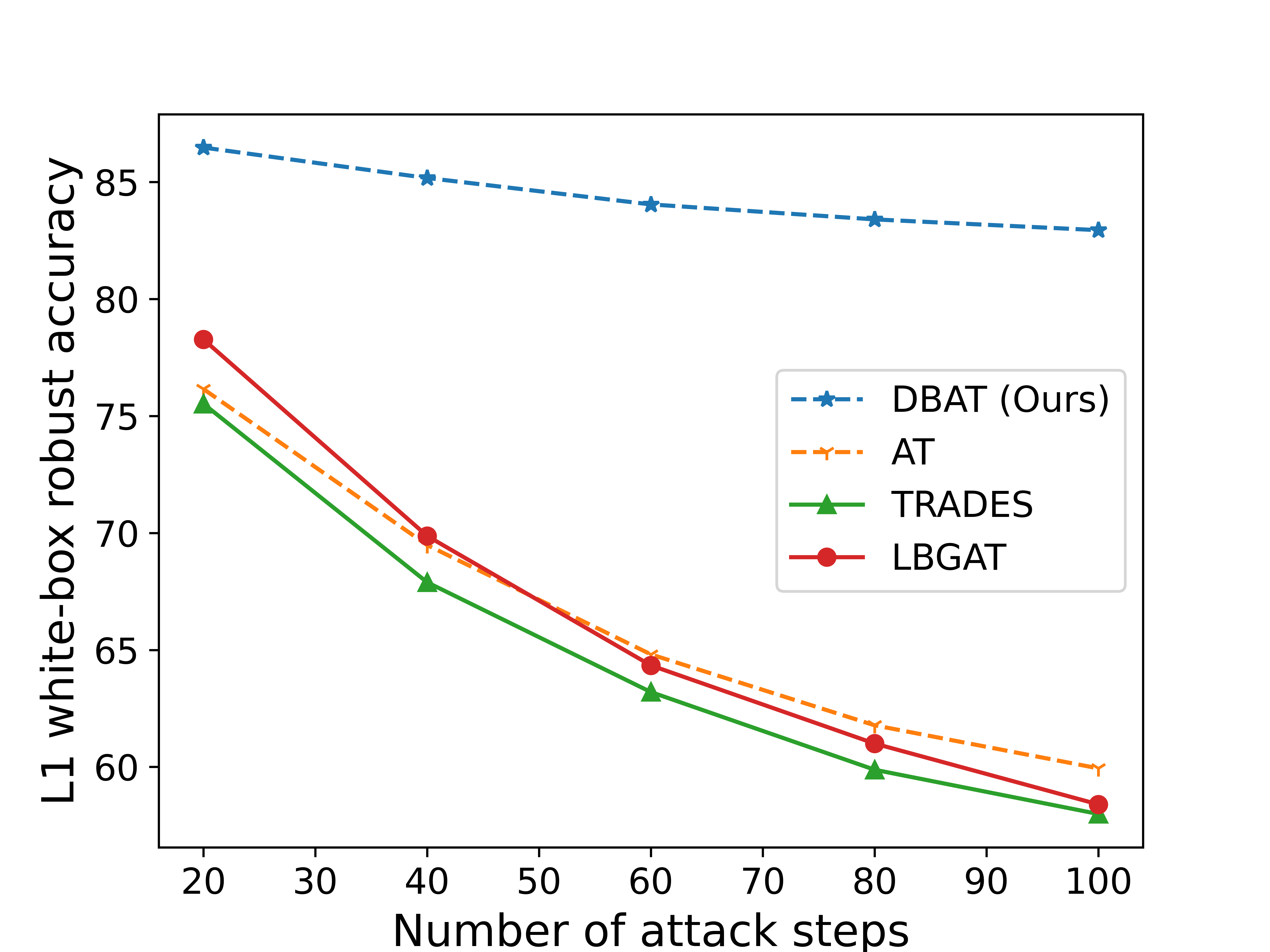}
    \caption{$\ell_{1}$-PGD}
  \end{subfigure}
  \begin{subfigure}{0.62\columnwidth}
    \centering
    \includegraphics[width=\linewidth]{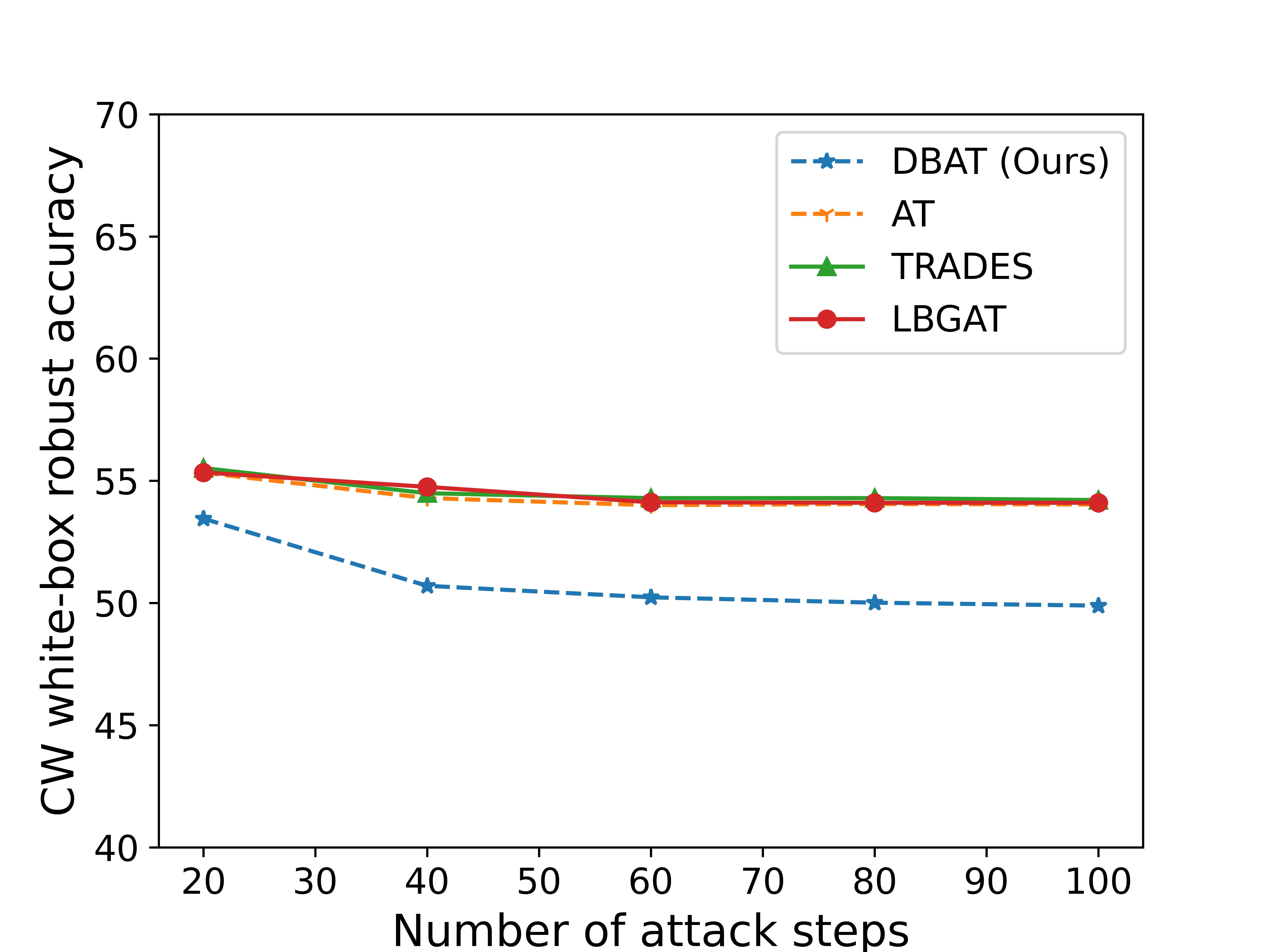}
    \caption{$CW_{\infty}$}
  \end{subfigure}
  \hspace{0.07\textwidth}
  \begin{subfigure}{0.62\columnwidth}
    \includegraphics[width=\linewidth]{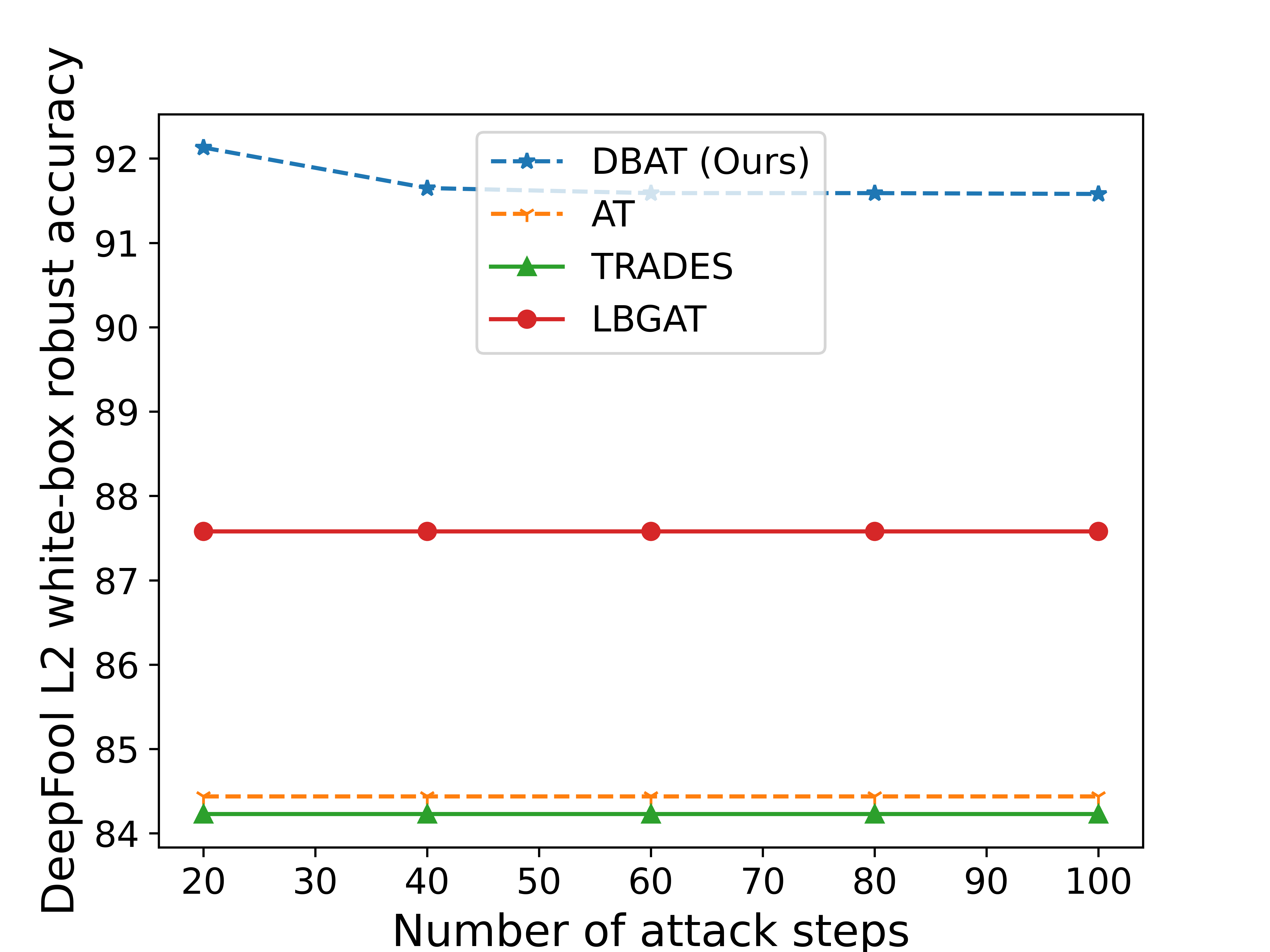}
    \caption{$\ell_{2}$-DeepFool}
  \end{subfigure}
  \begin{subfigure}{0.62\columnwidth}
    \includegraphics[width=\linewidth]{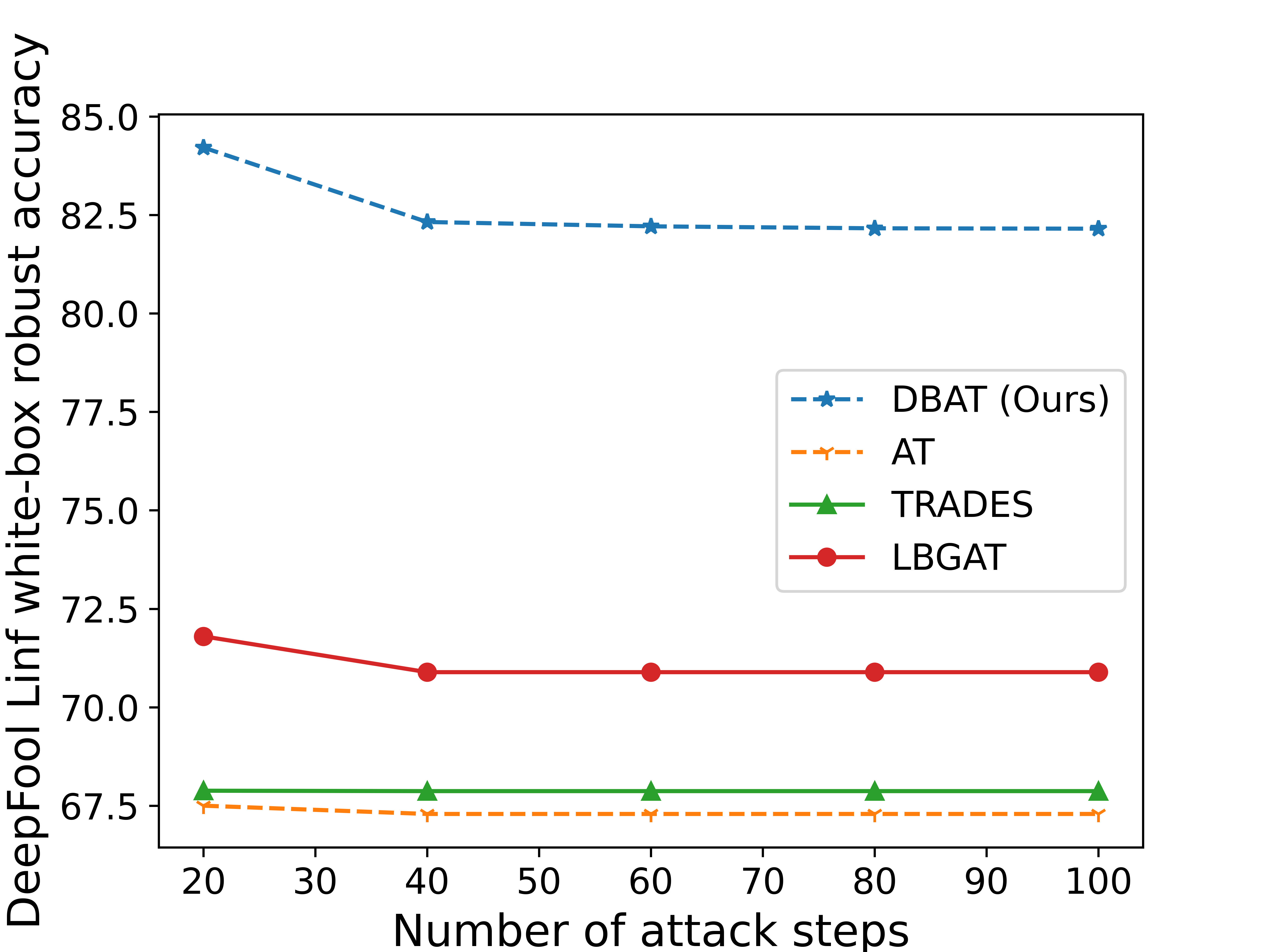}
    \caption{$\ell_{\infty}$-DeepFool}
  \end{subfigure}
  \caption{Robustness against \textbf{unforeseen} (a) $\ell_{2}$ PGD adversary (b) $\ell_{1}$ PGD adversary (c) $CW_{\infty}$ adversary 
  on CIFAR-10.}
  \label{unforeseen_adv}
\end{figure*}

\mycolor
Throughout the paper, we compare against the most powerful adversary.
In \cref{full-vs-standard-section}, we present additional experiments to demonstrate how different access to the inference time projection function affects the adversary's strength (i.e., the attacker's ability to degrade robust accuracy).
\normalcolor
As for black-box attacks analysis, we evaluate against two types of adversaries: naturally trained surrogate models, and other adversarially trained models. For natural corruptions, the corruptions are generated independently from the trained model.
\subsection{White/Black-box and Auto-Attack Evaluation} 
\paragraph{White-box/Black-box PGD Robustness.} 
We present DBAT's $l_{\infty}$-PGD white-box and black-box results 
compared to a variety of adversarial training methods. Attacks are generated with $\epsilon=8/255$, and perturbation step size 1/255 and 10 attack steps. Full numerical results and visualizations are in Appendix \ref{white-black-box-rob-cifar10}.
On CIFAR-10, DBAT's results are in line with the SOTA methods
under black-box attacks. For PGD white-box attacks, DBAT 
achieves significant PGD robustness (e.g., 54.25\% with PGD$_{1000}$), similar to the other methods, with near-optimal natural accuracy of 95.01\% (compared to 95.43\% for a naturally trained model).
Additionally, in Figures \ref{tsne} and \ref{tsne-emb} we visually present the strong class separation obtained by DBAT for the original classes, the newly generated adversarial classes, and the combination of all the 20 classes for CIFAR-10.

\paragraph{Auto-Attack Evaluation.} 
We evaluate DBAT on Auto-Attack, an ensemble of diverse attacks: APGD, APGD-DLR \cite{croce2020reliable}, Square \cite{andriushchenko2020square}, and FAB \cite{croce2020minimally}.
As described in Table \ref{nat_aa}, our method reaches near-optimal natural accuracy (compared to a naturally-trained model)
while still maintaining significant robustness when tested against AA. We note that Auto-Attack results are not as good as PGD results. We ascribe the difference to the adversarial classes that were generated using $\ell_{\infty}$-PGD and are therefore oriented towards PGD adversaries. 
It can be empirically evidenced in the ``unforeseen attacks'' (Figure \ref{unforeseen_adv}),
where our results on attacks such as C\&W
are good, but our results on the different PGD adversaries with different norms ($\ell_{\infty}, \ell_{2}, \ell_{1}$) are better.

\subsection{Unforeseen Adversaries Robustness}
To further demonstrate that our method does not suffer from false robustness, we test it against different adversaries that were not observed during training, including  $\ell_{2}$-PGD, $\ell_{1}$-PGD, $\ell_{\infty}$-DeepFool, and $\ell_{2}$-DeepFool \cite{moosavi2016deepfool} implemented by Foolbox \cite{rauber2017foolbox}, and CW$_{\infty}$ \cite{carlini2017towards}. We applied white-box attacks, with common attack budgets of 12 for $\ell_{1}$-PGD, 0.5 for $\ell_{2}$-PGD, 0.02 overshoot for DeepFool, and 
$8/255$ for $CW_{\infty}$. Results are visualized in Figure \ref{unforeseen_adv}
, and in Tables \ref{l2-res}, \ref{l1-res}, \ref{l2-deepfool-res}, \ref{linf-deepfool-res}, and \ref{cw-res} in Appendix \ref{unforeseen-adversaries}.
Our method significantly improves results (except for CW$_{\infty}$)
, even on unforeseen adversaries. %
\textit{DBAT improves $\ell_{2}$-PGD by up to 14\%, $\ell_{1}$-PGD by up to 20\%, $\ell_{2}$-DeepFool by up to 10\%, and $\ell_{\infty}$-DeepFool by up to 16\%}. 

\paragraph{Feature Adversaries.} 
We demonstrate the effectiveness of DBAT 
to so-called ``adaptive adversaries'':
those that try to circumvent our defense by ignoring the projection function during the optimization process  \cite{tramer2020adaptive}, and show that they cannot evade our defense.
To do so, we tested DBAT on CIFAR-10 against feature and logit-level adversaries: Kullback-Leibler divergence (KLD) attack on the probabilities vectors \cite{zhang2019theoretically}, $l_2$ logit-matching attack \cite{tramer2020adaptive} on adversarial examples and their corresponding natural examples, and lastly, a feature adversary suggested in \cite{sabour2015adversarial}.
We used $\epsilon=8/255$, $\delta=1/255$, and ran for 500 iterations.
Results are presented in Table \ref{feat-adv}. DBAT presents notably impressive and strong results against feature and logit-level adversaries. %
We also tried attacking the inner layers 
(and combinations of layers), 
in addition to attacking the feature representation layer, but we noticed it did not improve the attack success rate.
These results also support our claim in \cref{threat-model}, that an adaptive attacker will benefit from using the projection function in the attack optimization.

\begin{table}[ht]
\centering
\caption{CIFAR-10 results against feature and logit level adversaries.}
\label{feat-adv}
\begin{tabular}{cc}
Adversary & Robust Accuracy \\
\toprule
KLD  & 85.9 \\
$l_2$ Logit Matching & 84.5 \\
Feature Adversary \cite{sabour2015adversarial} & 86.8 \\
\bottomrule
\end{tabular}
\end{table}

\begin{figure}[ht]
\begin{center}
\centerline{\includegraphics[width=\columnwidth]{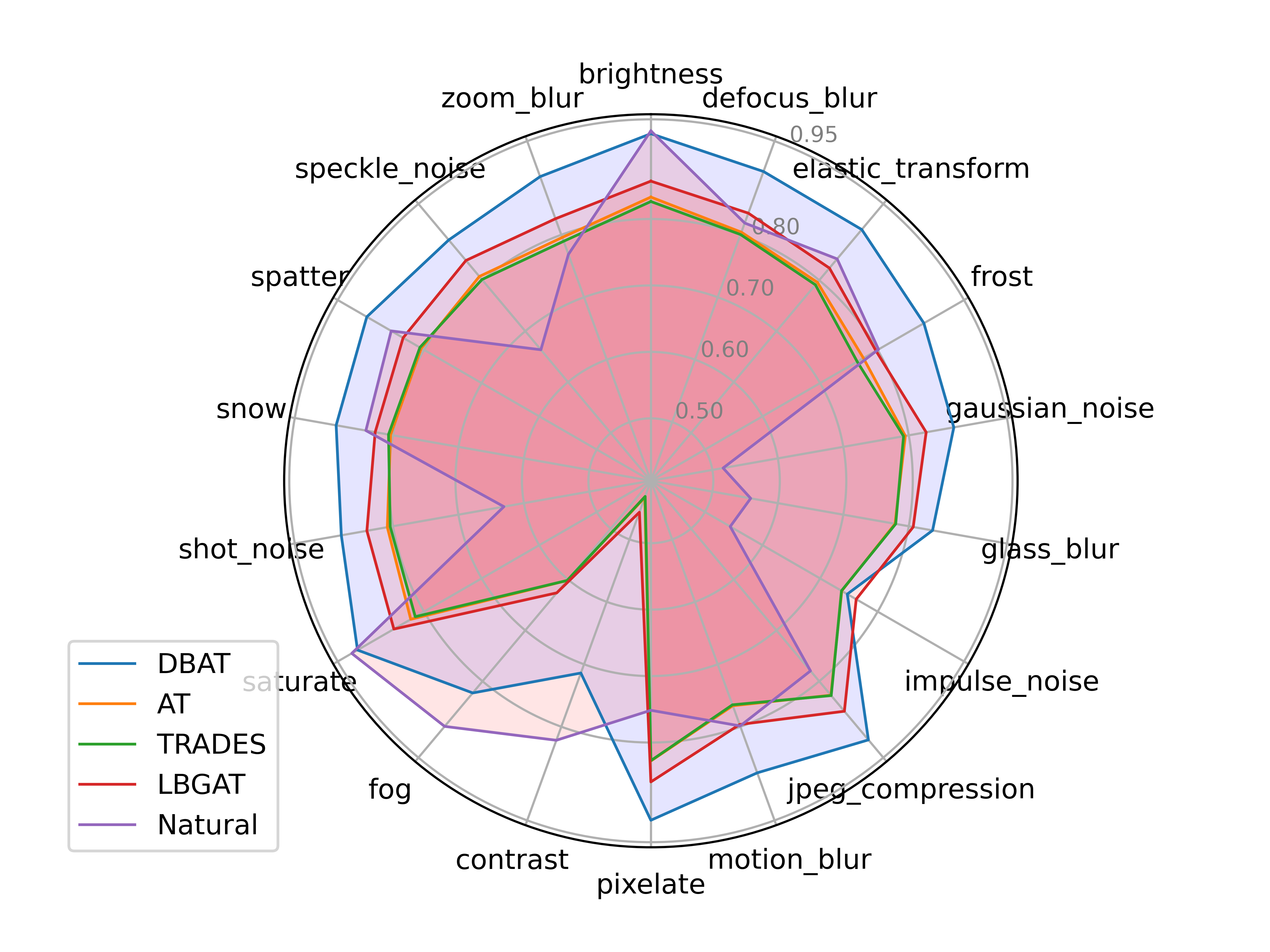}}
\caption{\mycolor CIFAR-10C accuracy comparison results between different methods over all 18 natural corruption types, including noise, blur, weather, and
digital categories.}
\label{radar-c}
\end{center}
\end{figure}

\begin{figure}[h]
\mycolorgreen
\begin{center}
\centerline{\includegraphics[width=\columnwidth]{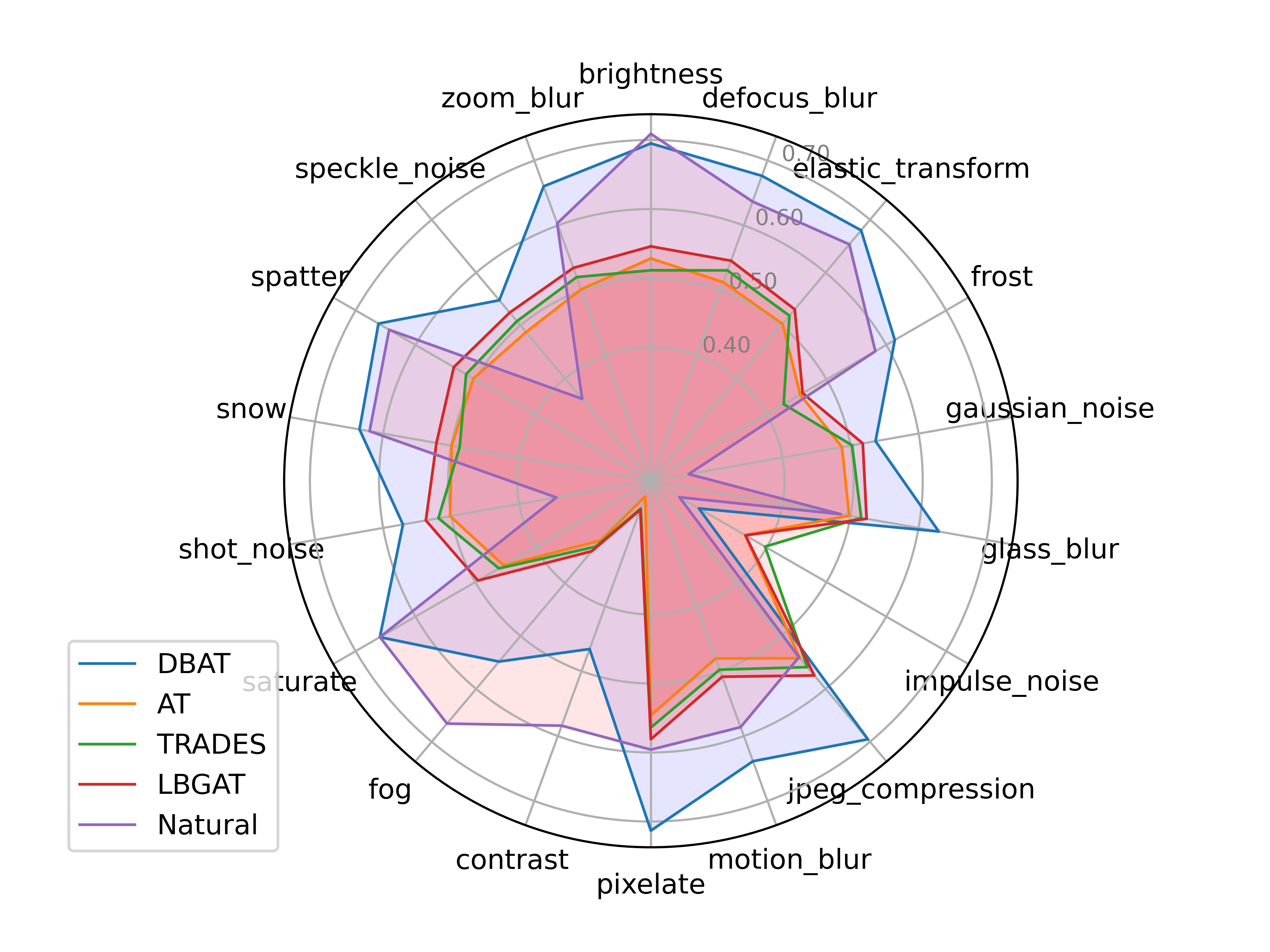}}
\caption{\mycolorgreen CIFAR-100C accuracy comparison results between different methods over all 18 natural corruption types, including noise, blur, weather, and
digital categories.}
\label{radar_cifar100c}
\end{center}
\normalcolor
\end{figure}

\subsection{Natural Corruptions Robustness}
\label{corruptions}
We demonstrate the effectiveness of DBAT when facing natural corruptions, as proposed by \cite{hendrycks2018benchmarking}. This corruptions benchmark dataset consists of 18 diverse corruption types. It covers noise, blur, weather, and digital categories. As the researchers claimed, research that improves performance on this benchmark should indicate general robustness gains, as the corruptions are varied and great in number. These corruptions each have five different levels of severity. To test DBAT, we use the CIFAR-10-C and CIFAR-100C corruptions benchmarks. Note that the corruptions are model-independent.
As demonstrated in Tables \ref{corruption-table1}, \ref{corruption-table2}, in Appendix \ref{natural-corruption-appendix}, and in Figures \ref{radar-c} and \ref{radar_cifar100c}, our method outperforms the other methods by a significant margin on all corruption types.

\paragraph{CIFAR-10C results.} When compared to the second-best performing method, \textit{DBAT obtains an average improvement of 7.96\% across all corruption types, and a maximum improvement of up to 35.19\%}.

\mycolorgreen
\paragraph{CIFAR-100C results.} When compared to the second-best performing method, \textit{DBAT obtains an average improvement of 10.82\% across all corruption types, and a maximum improvement of up to 25.75\%}.
\normalcolor

\mycolorgreen
\subsection{Clean vs. Robust Accuracy Trade-off}

\paragraph{Clean and robust accuracy trade-off.} Originally, the clean and adversarial classes were equally weighted during training. Meaning, given that $\lambda$ is the weighting factor for the adversarial classes, we set $\lambda=1$ in our experiments.

In the following experiment, we run an extensive evaluation to show how the trade-off between natural and robust accuracy changes as we weigh the loss on the natural and adversarial classes differently, i.e., how the natural and robust accuracy changes as we change the values of $\lambda$.
We use CIFAR-10 with the same experiment settings described above.
We report Auto-Attack (AA) results, as well as natural accuracy results.
In Figure \ref{tradeoff_exp} and Table \ref{lambda-tradeoff}, we plot DBAT's and TRADES's trade-off between natural and Auto-Attack robust accuracy as the weighting factor, $\lambda$, varies the trade-off between the natural and adversarial classes. For DBAT, we compare against the fully adaptive white-box perfect knowledge adversary (i.e., "Inference real-time access").
Not surprisingly, as we increase $\lambda$, clean accuracy decreases while robust accuracy increases, and vice-versa. However, as can be seen, the changes in natural accuracy for DBAT are relatively small, even though $\lambda$ was changed between a wide range of 0.1 and 8. 

\begin{figure}
\mycolorgreen
\begin{center}
\centerline{\includegraphics[width=\columnwidth]{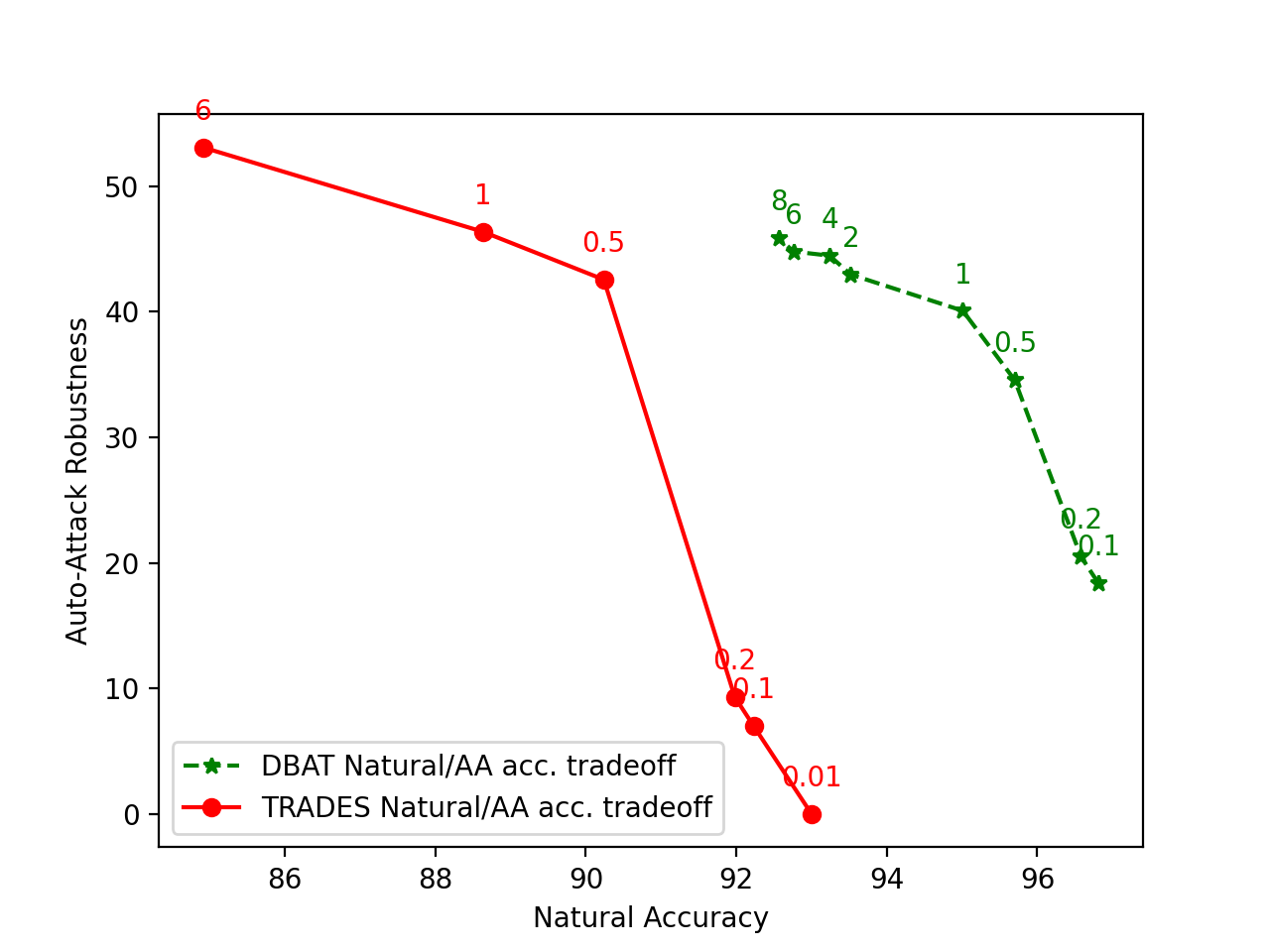}}
\caption{\mycolorgreen Natural and AutoAttack robust accuracy trade-off, for DBAT and TRADES on CIFAR-10, as we vary the hyper-parameter $\lambda$ that controls the weight we put on the natural and adversarial classes. The numbers on the graph represent the value of $\lambda$ for the specific trade-off.}
\label{tradeoff_exp}
\end{center}
\normalcolor
\end{figure}

\begin{table}[h]
\mycolorgreen
\centering
\caption{\mycolorgreen DBAT's natural and Auto-Attack accuracy trade-off on CIFAR-10 as the weighting factor $\lambda$ varies between 0.1 and 8 (where 1 is the default value used in the experiments).}
\begin{tabular}{c|c|c}
$\lambda$ & Natural        & Auto-Attack    \\
\toprule
0.1       & \textbf{96.81} & 18.40          \\
0.2       & 96.58          &  20.50         \\
0.5       & 95.70          & 34.56          \\
1         & 95.01          & 40.08          \\
2         & 93.52          & 42.97          \\
4         & 93.24          & 44.47          \\
6         & 92.76          & 44.80          \\
8         & 92.56          & \textbf{45.91} 
\label{lambda-tradeoff}
\end{tabular}
\end{table}

We also noticed that when increasing $\lambda$ beyond a value of 10, the model started to diverge. We attribute this behavior to the fact that we are over-weighting dynamically newly generated classes, which in turn affects the model's ability to converge. We noticed that a simple warm start of $\lambda$ can help to some extent. As for TRADES, we've noticed that as we decrease $\lambda$ below $0.1$, TRADES was not able to learn robust models.

Overall, we've demonstrated how reducing the trade-off parameter $\lambda$, TRADES was not able to match DBAT's clean accuracy without losing robust accuracy almost entirely. \textit{This is another empirical evidence of DBAT's unique ability to learn models with optimal or near-optimal natural accuracy and a significant level of robustness}.

\paragraph{F1-robust.} To further demonstrate that DBAT's robust-natural trade-off is indeed good compared to other methods, we adopt the recently proposed metric, F1-robust, suggested by \cite{levi2021domain} which was specifically designed as a balanced measurement for robust and natural accuracy. Results are presented in Table \ref{f1-robust}.

\begin{table}[ht]
\mycolorgreen
\centering
\caption{\mycolorgreen Results on CIFAR-10 using the recently proposed F1-robust metric specifically designed as a balanced measurement for robust and natural accuracy. Higher is better.}
\begin{tabular}{l|c}
Defense    & F1-Robust \\ 
\midrule
\underline{DBAT}       & \textbf{0.710}     \\
AT         & 0.657     \\
TRADES     & 0.659     \\
LBGAT      & 0.670      \\
Generalist & 0.685     \\
HAT        & 0.622      \\
UIAT       & 0.645    
\end{tabular}
\label{f1-robust}
\normalcolor
\end{table}

\normalcolor

\subsection{Generalization Across Datasets} 
\label{generalization}
To check the generalization of our approach to different datasets, we evaluate DBAT on SVHN and CIFAR-100. 
We present Auto-Attack results in Table \ref{nat_aa_svhn_cifar100}. Full results and visualizations are presented in 
Appendix 
\ref{cifar100-svhn-appendix}. As presented in Table \ref{nat_aa_svhn_cifar100}, our method reaches optimal (for SVHN) and near-optimal (for CIFAR-100) natural accuracy on the different datasets while
still maintaining significant robustness even against Auto-Attack.
\paragraph{SVHN evaluation.} For SVHN, DBAT achieves an improvement of 3-5\% under black-box attacks, When tested against white-box PGD attacks, DBAT achieves significant PGD robustness of 53.40\% (PGD$_{20}$), similar to other popular AT methods. See Table \ref{nat_aa_svhn_cifar100} and Appendix \ref{cifar100-svhn-appendix} for full results.
Additionally, DBAT reaches a natural accuracy of 96.86\%, compared to 96.85\% for a naturally trained model under the same architecture and settings.
Meaning, \textit{\textbf{DBAT exhibits no reduction in natural accuracy}, while also achieving significant robust accuracy
under various strong adversaries}.
\begin{table}[h]
\mycolor
\caption{Natural, PGD
, and Auto-Attack (AA) robust accuracy against the fully adaptive white-box perfect knowledge adversary (i.e., "Inference real-time access"). The attack is an $\ell_{\infty}$-PGD attack on SVHN and CIFAR-100. Similar to our method, the presented results are for models that \textbf{do not utilize additional data} during training. The Natural method refers to a model trained using standard training under the same hyper-parameters settings.
Green represents the best natural accuracy among the robust models. Orange represents the second-best natural accuracy among the robust models. The range alongside the green arrow represents the natural accuracy improvement compared to the other methods.}
\label{nat_aa_svhn_cifar100}
\begin{center}
\begin{small}
\begin{sc}
\resizebox{\columnwidth}{!}{%
\begin{tabular}{llccc}
\toprule
Dataset & Method & Natural Acc. & PGD & AA \\
\midrule
\multirow{9}{*}{CIFAR-100} 
& \underline{DBAT}  & {\color{Green} \textbf{75.18} (\greenup 12.2--18.5\%)} & 27.22 & 18.17\\
& AT & 56.73 & 28.45 & 24.12 \\
& TRADES & 58.24 & 29.70 & 24.90 \\
& LBGAT &  60.64 & 34.84 & 29.33 \\
& Generalist & {\color{YellowOrange}62.97} & 29.49 & 23.96 \\
& HAT &  58.73 &  27.92 & 23.34 \\
& UIAT & 59.55 & 30.81 & 25.73 \\
& CAT & 62.84 & - & 16.82 \\
\cmidrule{2-5}
& Natural & 79.30 & 0 & 0 \\
\midrule
\midrule
\multirow{9}{*}{SVHN} 
& \underline{DBAT} & {\color{Green} \textbf{96.86} (\greenup 2.8--6.8\%)} & 49.31 & 40.49 \\
&AT & 89.90 & 49.45& 45.25 \\
& TRADES & 90.35 & 54.13 & 49.50 \\
& LBGAT & 91.80 & 63.38 & 40.83 \\
& Generalist & {\color{YellowOrange} 94.11} & 55.29 & 45.41 \\
& HAT &  92.06 & 57.35 & 52.06\\
& UIAT & 93.28 & 58.18 & 52.45 \\
& CAT & - & - & - \\
\cmidrule{2-5}
& Natural & 96.85 & 0 & 0 \\
\bottomrule
\end{tabular}
}
\end{sc}
\end{small}
\end{center}
\normalcolor
\end{table}
\begin{table*}[h]
\caption{Natural and Auto-Attack (AA) $\ell_{\infty}$ robustness on CIFAR-10, CIFAR-100 and SVHN, against adversaries with different capabilities. \textit{Model parameters + projection function access} refers to the adversary capabilities to access both model parameters and inference projection function. \textit{inference real-time access} refers to the adversary that 
can
access not only to the entire network parameters but also to the
defender's system and projection function at any given time
during inference.
\textit{model parameters} refers to an adversary that gains access only to the model parameters, but not to the inference time projection function.}
\label{full-vs-standard}
\begin{center}
\begin{small}
\begin{tabular}{llcc}
\toprule
Dataset  & Adversary access capabilities & Natural Acc. & Robust Acc. \\
\midrule
\multirow{3}{*}{CIFAR-10}  & Model parameters access &  \multirow{3}{*}{95.01} &  \textbf{50.31} \\
& Model parameters + projection function access & &  47.82 \\
&  Inference real-time access (model params. + inference-time access to projection func.) &   &  40.08 \\
\midrule
\midrule
\multirow{3}{*}{CIFAR-100} 
& Model parameters access & \multirow{3}{*}{75.18} &  \textbf{23.16} \\
& Model parameters + projection function access & &  20.87\\
& Inference real-time access (model params. + inference-time access to projection func.) & & 18.17\\
\midrule
\midrule
\multirow{3}{*}{SVHN} 
 & Model parameters access & \multirow{3}{*}{96.86} & \textbf{56.60} \\
& Model parameters + projection function access & & 48.58 \\
& Inference real-time access (model params. + inference-time access to projection func.) &  & 40.49 \\
\bottomrule
\end{tabular}
\end{small}
\end{center}
\end{table*}
\paragraph{CIFAR-100 evaluation.} 
For CIFAR-100, DBAT achieves an improvement of 6-12\% under black-box attacks.
When tested in PGD white-box attacks, DBAT still achieves significant PGD robustness, e.g., 29.95\% with PGD$_{20}$, similar to other popular AT methods.
We note that Auto-Attack robustness is lower than the other methods, possibly due to the greater diversity in the dataset and the small number of examples in each class, which makes it more difficult to learn new adversarial class boundaries. See Appendix \ref{cifar100-svhn-appendix} for full results.
DBAT achieves significant Auto-Attack robustness 
. Moreover, DBAT reaches a natural accuracy of 75.18\%, compared to 79.30\% for a naturally trained model under the same settings. \textit{Compared to the other methods, DBAT improves natural accuracy by 14.5-18.5\%}.

\subsection{Ablation Studies}
\paragraph{DBAT core components.} We demonstrate the performance gain obtained by our method by removing the two parts that are not at the core of DBAT -- SWA and Cutout. We use the CIFAR-10 dataset, with WRN-34-10, and report the Auto-Attack (AA) results when removing SWA and Cutout. 
When removing SWA and Cutout
we observe that their total contribution to DBAT is 2.21\% in natural accuracy and 3.55\% in robust accuracy.
Additionally, in Table \ref{diff-agg} we present the results using different aggregation functions (sum and mean).
Another study that tests the effect of training with targeted versus untargeted PGD is presented in Appendix \ref{targeted-pgd}.
Altogether, we conclude that the majority of the gain in natural and robust accuracy is obtained by DBAT.

\mycolorgreen
\paragraph{Numerical instability.} The loss function calculates the log over the max on the Softmax probabilities. This additional log may cause numerical instability, as discussed in Appendix G.2 of \cite{tramer2020adaptive}. To demonstrate that our method does not suffer from numerical instability, we conducted two additional experiments:
\begin{itemize}
  \item We replaced the max with LogSumExp (LSE) which should be more stable. Changes in results were within a standard deviation of $\mp0.22$ from the original reported results.
  \item We ran both AA and PGD-20 with 5 random restarts (within epsilon) and calculated mean and std. All results were within a standard deviation of $\mp0.2$.
\end{itemize}

\paragraph{Model complexity and training time overhead.} We acknowledge the fact that DBAT presents additional complexity to the model. However, keeping in mind that only the output of the last final fully connected layer is doubled, the additional model complexity is minor in most of the cases. Specifically, with WRN-32-10 on CIFAR-100, DBAT introduces 64k additional parameters, which sums up to an additional 0.13\% of the total parameters. Additionally, we also analyzed the training time overhead of our approach compared to the well-known TRADES. Overall, DBAT has a minor overhead of up to 2-3\% for models with class numbers ranging from 10 to 100.
\normalcolor

\begin{table}[h]
\centering
\small
\caption{Natural and Auto-Attack (AA) results against the strongest, \textit{inference real-time access} adversary, using two additional aggregation functions: sum and mean.}
\label{diff-agg}
\begin{tabular}{ccccc}
Agg. function & Acc. & CIFAR-10 & SVHN & CIFAR-100 \\ 
\toprule
\multirow{2}{*}{Sum} & Natural & 95.01 & 96.85 & 75.07 \\
 & AA & 39.81 & 36.00 & 18.08 \\ 
 \midrule
\multirow{2}{*}{Mean} & Natural & 95.01 & 96.85 & 75.07 \\
 & AA & 39.93 & 35.84 & 18.49 \\
 \bottomrule
\end{tabular}
\end{table}

\subsection{Adaptive vs. Non-adaptive Attacks}
\label{full-vs-standard-section}
In \cref{threat-model} we stated that the most powerful white-box adversary 
is one
who has access to both model parameters and to the projection function at any given time during inference. 
This adversary, which we used throughout the paper and termed \textit{inference real-time access} adversary, 
possesses all possible capabilities and may also be thought of as a ``Perfect-Knowledge Adversary''. 
That is, the adversary has access to the model parameters, and more importantly, 
to the projection function the defender is using at each given time during inference ---
and thus can utilize the same projection function while attacking.
Throughout the paper, we compare primarily against this unrealistically powerful adaptive adversary.
Although this adversary is not the most realistic one, here we wish to demonstrate that this kind of adversary is indeed
the most powerful adversary. 
To do so, we compare the results with two different white-box adversaries. 
The first adversary will be referred to as the \textit{model parameters} adversary.
It is assumed to possess
access to the model parameters, but not to the inference projection function ---
meaning that the adversary can optimize the network parameters, but is not utilizing the inference projection function (e.g., max) in the attack optimization process. 
The \textit{model parameters}
adversary illustrates how the attack success rate is influenced by the attacker's adaptive knowledge 
(or lack thereof)
about the defender's projection function.
The second, most realistic adversary, termed \textit{model parameters + projection function access} adversary, knows that the defender is utilizing a projection function, but does not have inference-time access to the defender's choices at any given time during inference, and therefore has to conjecture the projection function from common projection functions (mean, softmax, maximum, sum, etc.) while attacking.
Both adversaries --- \textit{model parameters} and \textit{model parameters + projection function access} ---
are more realistic than the \textit{inference real-time access} adversary, since the projection function is not part of the optimization. For example, the trained model can be published at model zoos, while the projection function does not. Alternatively, the defender can randomize or change the projection function at any time during inference.

Table \ref{full-vs-standard} presents results under the three adversaries' settings.
As can be seen, since the \textit{model parameters} adversary does not utilize the knowledge about the projection function,
the attack against DBAT becomes much less effective, and as a consequence the attack success rate decreases, i.e., model robustness increases.


\section{Conclusion}

In this paper, we 
demonstrate 
the advantage of treating the clean and adversarially perturbed examples
as belonging to separate classes, instead of insisting that the classifier ``stretch''
a single class to accommodate them both.
With this new idea in mind, we proposed Double Boundary Adversarial Training (DBAT). Our extensive evaluation illustrates the ability of DBAT to achieve state-of-the-art results under various tasks such as 
black-box PGD attacks, natural corruptions robustness, and unforeseen adversaries (e.g., $\ell_{2}$-PGD, $\ell_{1}$-PGD, and DeepFool). 
That said, our aim is not to compete with the state-of-the-art in robustness across the board. Rather, we wish to equip models with a significant level of robustness,  while only incurring a minor or negligible degradation to their original natural accuracy.
Therefore, the main benefit of DBAT is its ability to reach optimal or near-optimal natural accuracy while achieving significant robustness, even against strong adversaries. This ability makes DBAT applicable for real-world applications (e.g., healthcare, autonomous vehicles, and security systems) that cannot sacrifice much of their natural accuracy.

\section{Ethical Considerations}

The existence of adversarial examples points to a basic weak-
ness of deep neural networks. These kinds of attacks were proven to break state-of-the-art networks in different fields such as Computer Vision (CV), Natural Language Processing (NLP), and many others. With the deployment of AI in safety-critical systems, such as security systems, medical diagnosis, and autonomous driving, it is at a premium to build systems that are robust, at least to some extent, against such attacks.
However, the gain of robustness is usually at the expense of the systems' natural accuracy. In order for real-world, safety-critical systems to adopt robust models, we need to make sure that the degradation in the natural accuracy is minor. For this reason, we suggested DBAT, which achieved a significant level of robustness without sacrificing much of the natural accuracy, and hope that it will help real-world applications adopt robust models.

Having said that, DBAT still has its limitations: adversarial training is an expensive training method that requires extra computations when compared to vanilla training. Moreover, DBAT achieves significant level of robustness, but it does not eliminate it completely, and one should take it under consideration when deploying such methods.
Overall, it is important to remember that the general problem
of robust models is still far from being fully solved, and the research community is heading a long way to go until we will manage to achieve sustainable robustness in real-world scenarios.


\bibliographystyle{plain}
\bibliography{egbib.bib}


\appendix

\begin{figure*}
  \centering
  \begin{subfigure}{0.3\linewidth}
    \includegraphics[width=\linewidth]{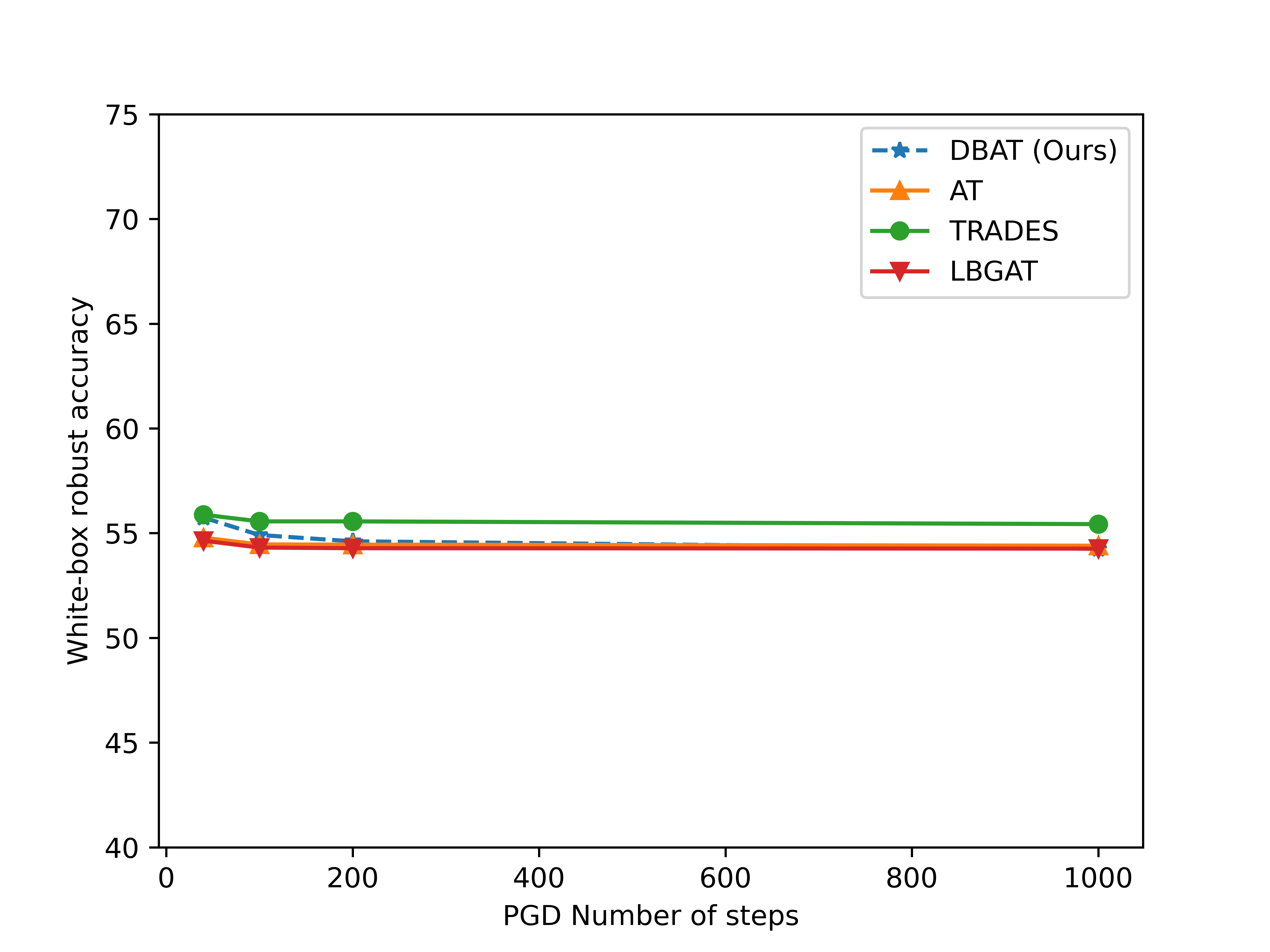}
    \caption{CIFAR-10 White-box}
  \end{subfigure}
  \hspace{0.03\textwidth}
  \begin{subfigure}{0.3\linewidth}
    \includegraphics[width=\linewidth]{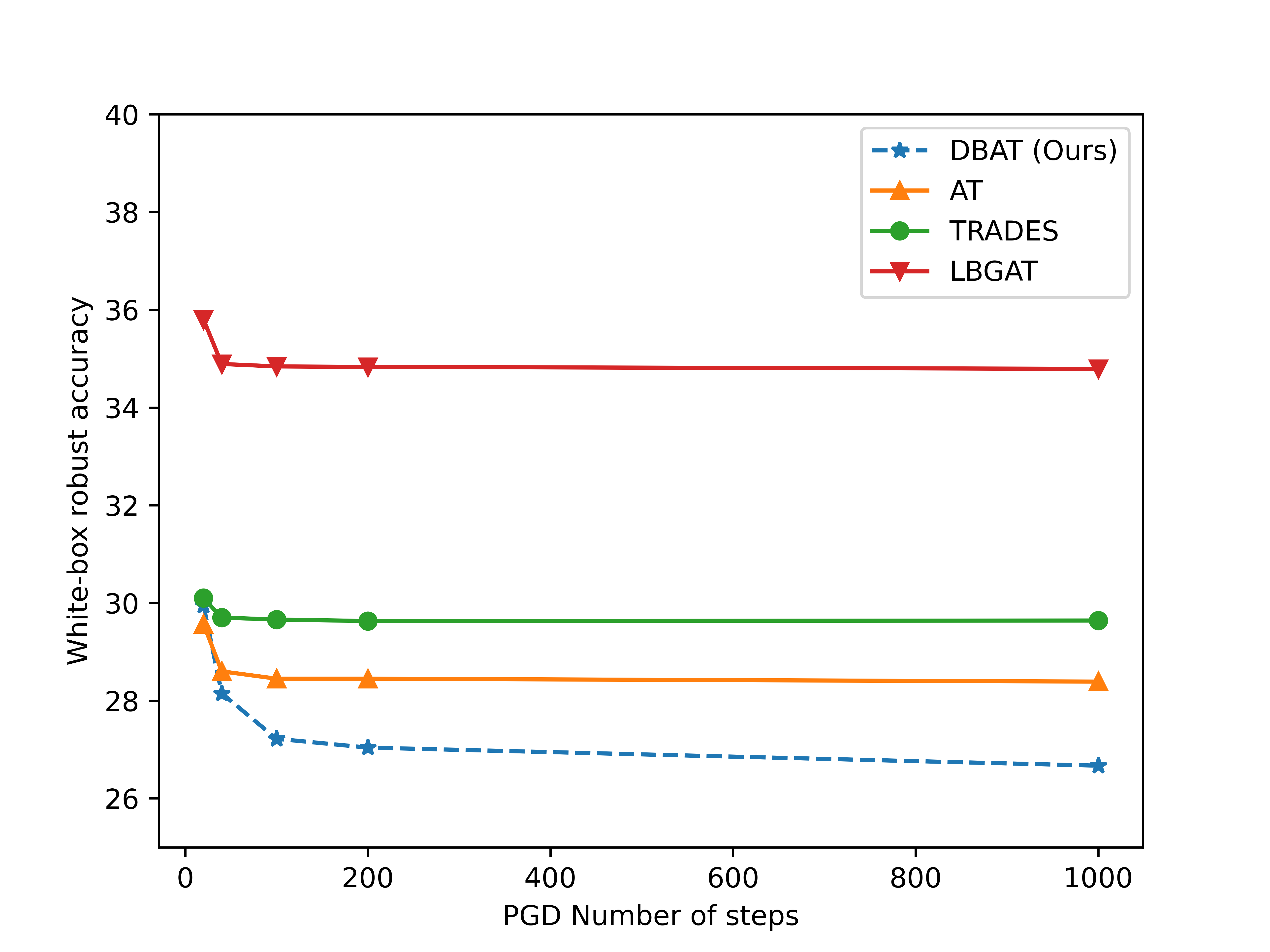}
    \caption{CIFAR-100 White-box}
  \end{subfigure}
  \hspace{0.03\textwidth}
  \begin{subfigure}{0.3\linewidth}
    \includegraphics[width=\linewidth]{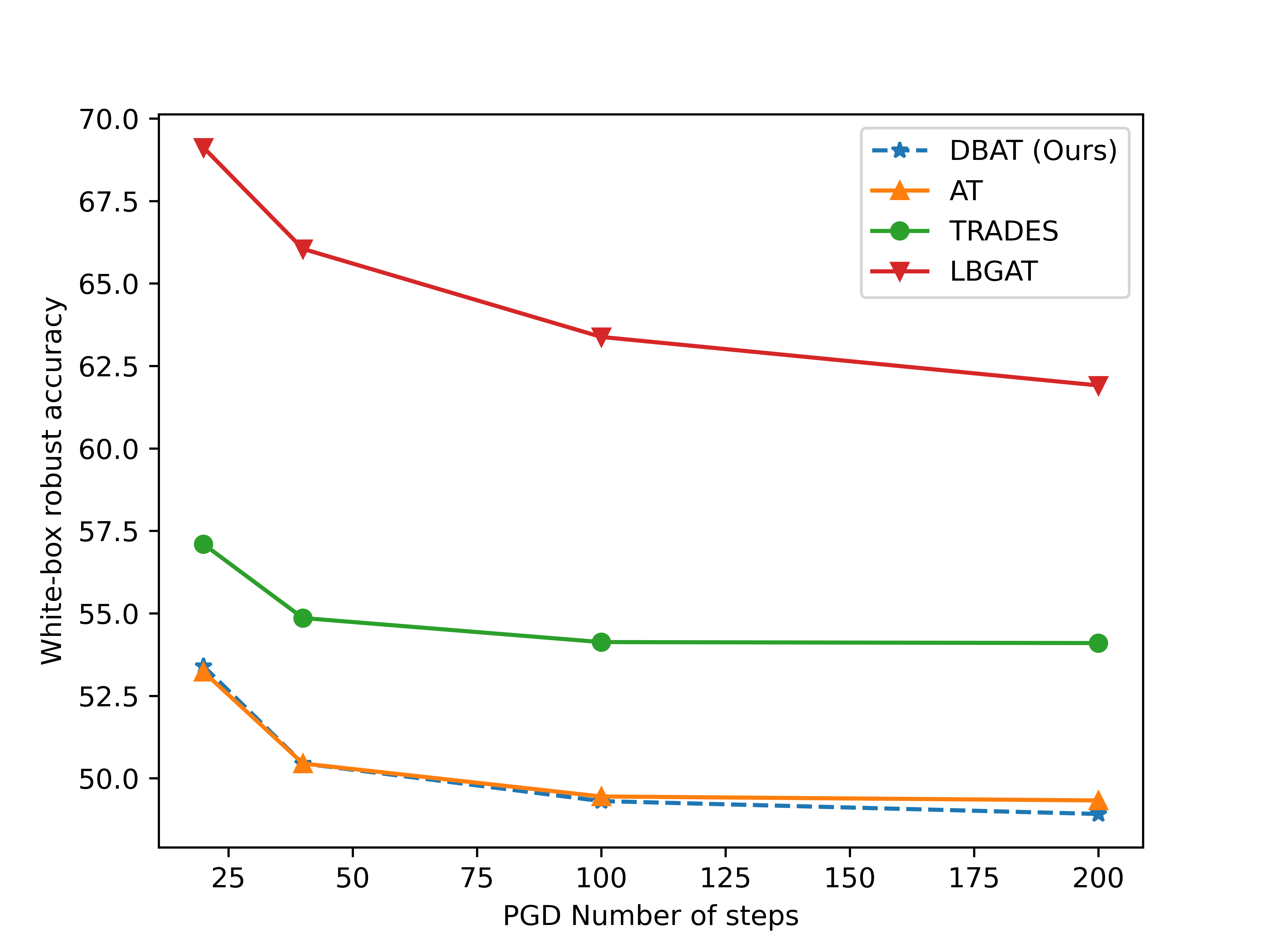}
    \caption{SVHN White-box}
  \end{subfigure}
  
  \caption{White-box PGD robustness on (a) CIFAR-10 (b) CIFAR-100, and (c) SVHN.}
  \label{all-pgd-white-box}
\end{figure*}

\begin{figure*}
  \centering
  \begin{subfigure}{0.3\linewidth}
    \includegraphics[width=\linewidth]{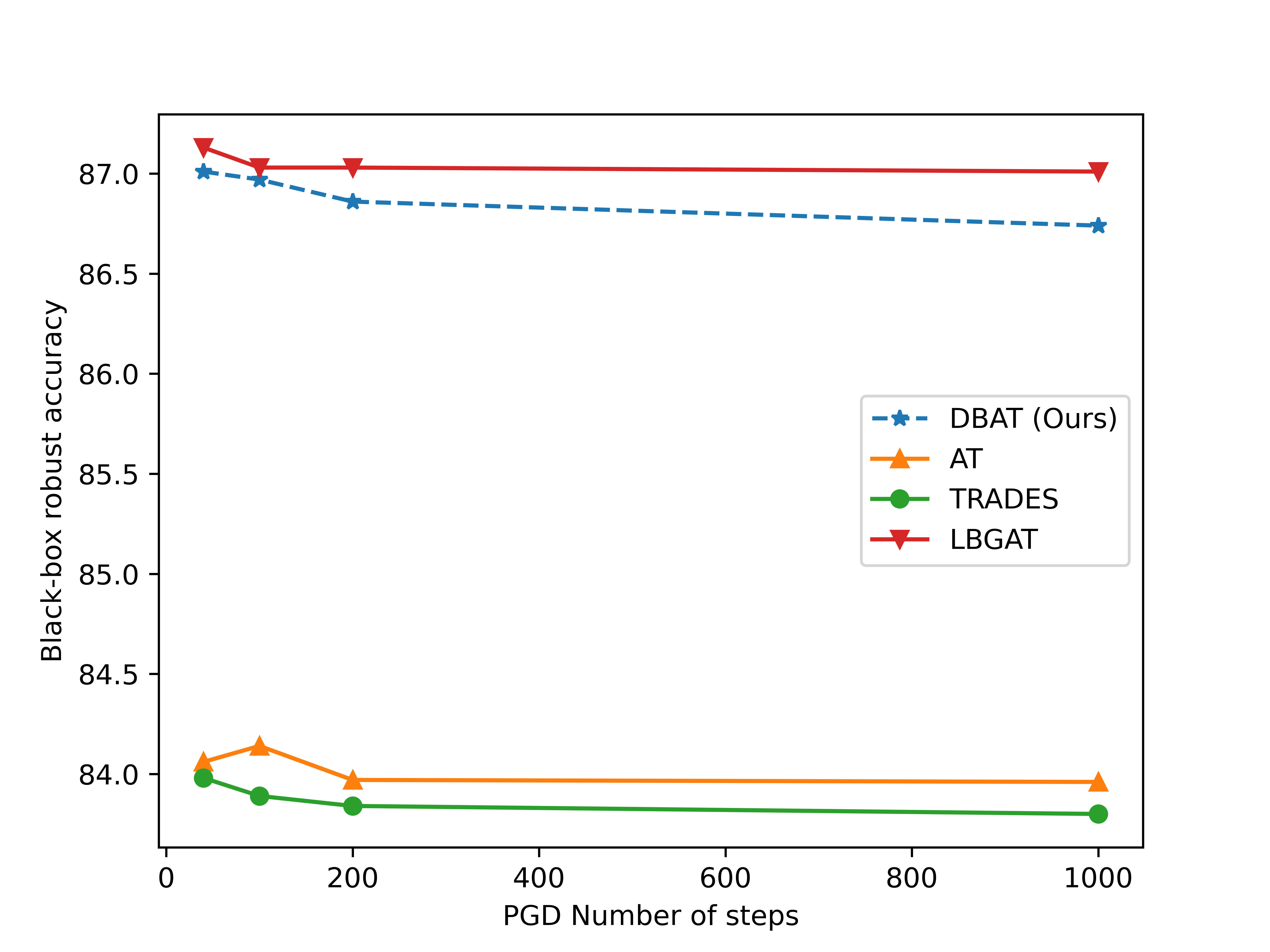}
    \caption{CIFAR-10 Black-box}
  \end{subfigure}
  \hspace{0.03\textwidth}
  \begin{subfigure}{0.3\linewidth}
    \includegraphics[width=\linewidth]{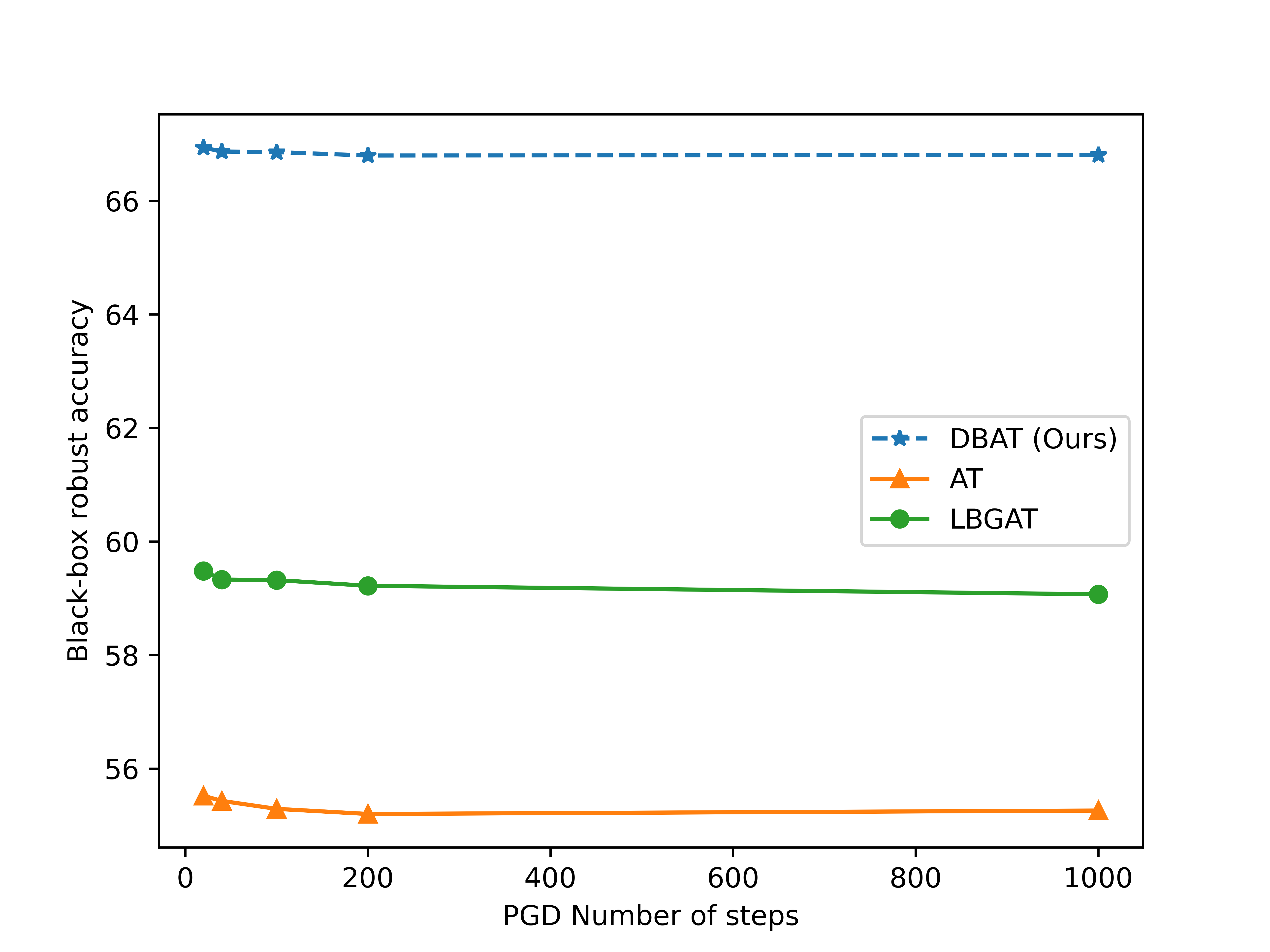}
    \caption{CIFAR-100 Black-box}
  \end{subfigure}
  \hspace{0.03\textwidth}
  \begin{subfigure}{0.3\linewidth}
    \includegraphics[width=\linewidth]{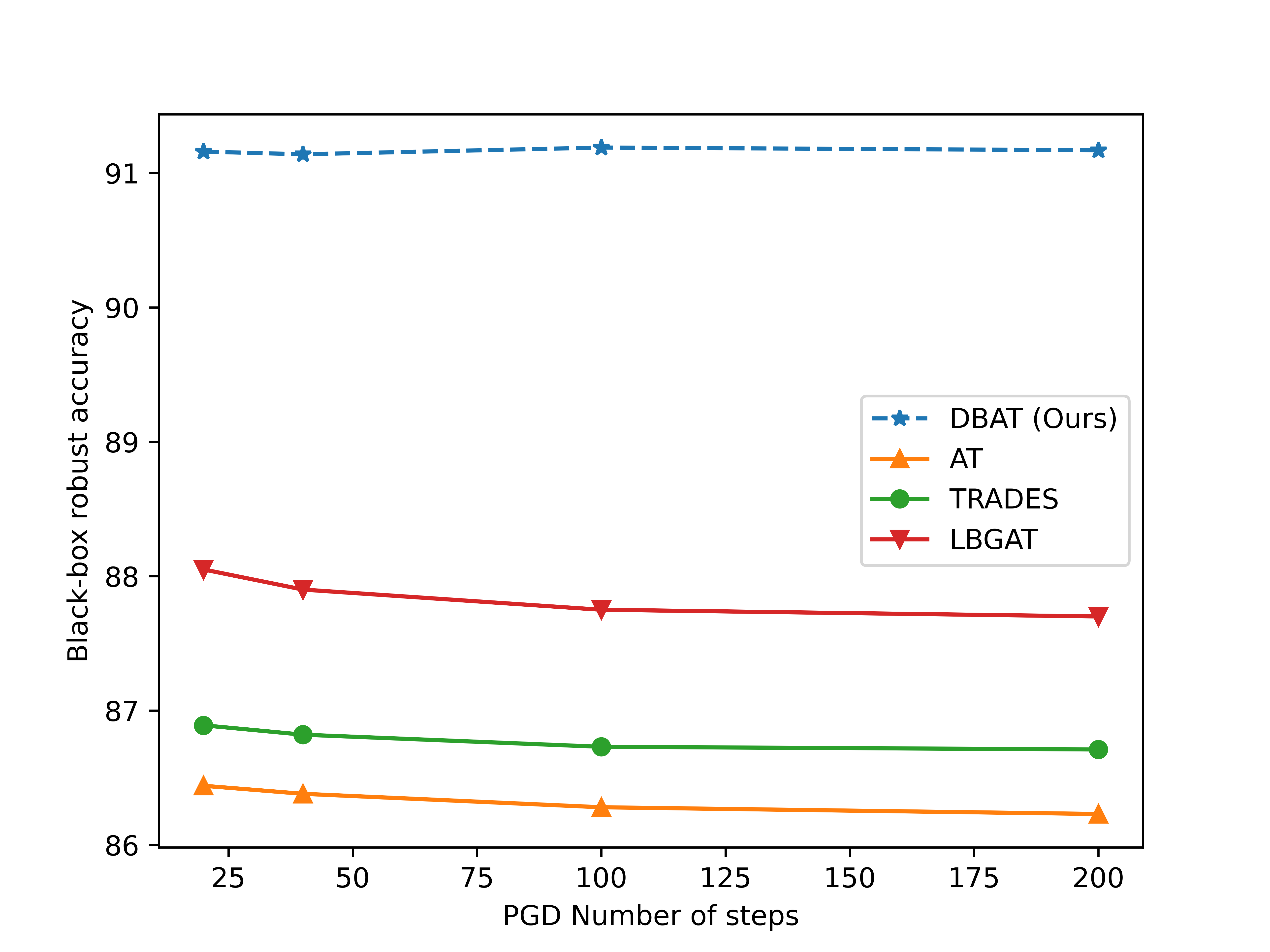}
    \caption{SVHN Black-box}
  \end{subfigure}
  
  \caption{Black-box PGD robustness on (a) CIFAR-10 (b) CIFAR-100, and (c) SVHN.}
  \label{all-pgd-black-box}
\end{figure*}

\begin{table*}[h]
  \caption{White-box robustness against on CIFAR-10.}
  \label{white-box-cifar10-res}
  \centering
  \begin{tabular}{lllllll}
    \toprule
    Defense Model & Natural & PGD$^{20}$ & PGD$^{40}$ & PGD$^{100}$ & PGD$^{200}$ & PGD$^{1000}$ \\
    \midrule
    DBAT (Ours) & \textbf{95.01} &  58.10 & 55.74 & 54.91 & 54.61 & 54.25  \\
    AT & 85.10 & 56.28 & 54.78 & 54.46 & 54.45 & 54.40  \\
    TRADES & 84.92 & 56.60 & 55.88 & 55.56 & 55.56 & 56.43  \\
    LBGAT & 88.22 & 55.89 & 54.64 & 54.31 & 54.28 & 54.26  \\
    \bottomrule
  \end{tabular}
\end{table*}

\begin{table*}[ht]
  \caption{Robustness against black-box attacks on CIFAR-10. Black-box adversary is a naturally trained surrogate model.}
  \centering
  \begin{tabular}{lllllll}
    \toprule
    Defense Model & Natural & PGD$^{20}$ & PGD$^{40}$ & PGD$^{100}$ & PGD$^{200}$ & PGD$^{1000}$ \\
    \midrule
    DBAT (Ours) & \textbf{95.01} & \textbf{87.00} & \textbf{87.01} & \textbf{86.97} & \textbf{86.86} & \textbf{86.74}   \\
    AT & 85.10 & 84.22 & 84.06 & 84.14 & 83.97 & 83.96   \\
    TRADES & 84.92 & 84.08 & 83.98 & 83.89 & 83.84 & 83.80   \\
    LBGAT & 88.22 & 87.23 & 87.13 & 87.03 & 87.03 & 80.01 \\
    \bottomrule
  \end{tabular}
\label{black-box-cifar10-res}
\end{table*}

\section{Full Experimental Setup Details}
\label{exp-details}
We conduct our experiments on CIFAR10, CIFAR100 \cite{krizhevsky2009learning} and SVHN \cite{netzer2011reading}.
We use the wide residual network (WRN-34-10) \cite{zagoruyko2016wide} architecture for CIFAR-10 and CIFAR-100, and the PreAct ResNet-18 \cite{he2016identity} for SVHN. Following \cite{rice2020overfitting}, we apply early-stopping based on a validation set.
The batch size is set to 128, weight decay is set to $7e^{-4}$. We train the model for 200 epochs with an initial learning rate of 0.1. For CIFAR datasets, the learning rate is decayed by a factor of 10 at iterations 50 and 150. For SVHN, the learning rate is decayed by a factor of 10 at iterations 50 and 75. 
Natural images are padded with 4-pixel padding with 32-random crop and random horizontal flip. Furthermore, all methods are trained using SGD with momentum 0.9. 

We combine Stochastic Weight Averaging (SWA) \cite{izmailov2018averaging}, exponential moving average on the model weights during training steps, within our training process. SWA was shown to be effective in training robust models \cite{chen2020robust, rebuffi2021data}, due to its temporal ensemble effect, and the ability to smooth the weights. Additionally, we add Cutout \cite{devries2017improved} data augmentation with a window length of 8.

\section{Additional Results for White-box/Black-box Robustness on CIFAR-10}
\label{white-black-box-rob-cifar10}
Full numerical results for CIFAR-100 are presented in Tables \ref{white-box-cifar10-res}, \ref{black-box-cifar10-res}, and \ref{black-box-cifar-adv}. Attacks are generated using $\ell_{\infty}$-PGD with $\epsilon=8/255$, and perturbation step size 1/255.
Black-box attacks in Table \ref{black-box-cifar10-res} were generated using a naturally trained surrogate model, and using adversarial surrogate models in Tables \ref{black-box-cifar-adv} and \ref{black-box-cifar-adv-same-at}.

In Figure \ref{black-box-cifar-adv-same-at}, we show that our method is also robust against a black-box adversary that has knowledge about our architecture and projection procedure, and is able to train a similar surrogate model. We compare our results under an adversary with the same capabilities (knowledge about architecture) against TRADES and AT, and show that our method can maintain similar robustness to the other well-known methods while reaching near-optimal natural accuracy.

\begin{table*}[ht]
  \caption{Black-box attack results with DBAT and adversarially trained surrogate models on CIFAR-10.}
  \label{black-box-cifar-adv}
  \centering
  \small
  \begin{tabular}{ll|ccccc}
    \toprule
    \cmidrule(r){1-2}
    Surrogate model & Target model & Natural (Target) & PGD$^{20}$ & PGD$^{40}$ & PGD$^{100}$ & PGD$^{200}$ \\
    \midrule
    AT & DBAT & \textbf{95.01} & \textbf{67.75} & \textbf{67.13} & \textbf{66.86} & \textbf{66.86} \\
    DBAT & AT & 85.10 & 65.55 & 64.54 & 64.36 & 64.29 \\
    \midrule
    TRADES & DBAT & \textbf{95.01} & \textbf{68.12} & \textbf{67.21} & \textbf{67.04} &  \textbf{67.03}  \\
    DBAT & TRADES & 84.92 & 66.42 & 65.86 & 65.75 & 65.71  \\
    \bottomrule
  \end{tabular}
\end{table*}

\begin{table*}[ht]
  \caption{Black-box attack results where each surrogate and target models are a pair of separately trained models using the same adversarial training method and same architecture.}
  \label{black-box-cifar-adv-same-at}
  \centering
  \small
  \begin{tabular}{ll|ccccc}
    \toprule
    \cmidrule(r){1-2}
    Surrogate model & Target model & Natural & PGD$^{20}$ & PGD$^{40}$ & PGD$^{100}$ & PGD$^{200}$ \\
    \midrule
    DBAT (Ours) & DBAT (Ours) & \textbf{95.01} & 69.25 & 68.30 & 68.11 & 68.07 \\
    TRADES & TRADES & 84.92 & 66.15 & 65.64  & 65.59 &  65.58 \\
    AT & AT & 85.10 & 67.70 & 67.27 & 67.24 & 67.24  \\
    \bottomrule
  \end{tabular}
\end{table*}

\section{Additional Classes Visualizations on the Features Space}
\label{vls-viz2}
In Figure \ref{tsne-emb}, we visually present the strong class separation obtained by DBAT for the original classes, the newly generated adversarial classes, and the combination of all the 20 classes for CIFAR-10 on the \textbf{embedded feature space}.

\begin{figure*}
  \centering
  \begin{subfigure}{0.25\linewidth}
    \includegraphics[width=0.97\linewidth]{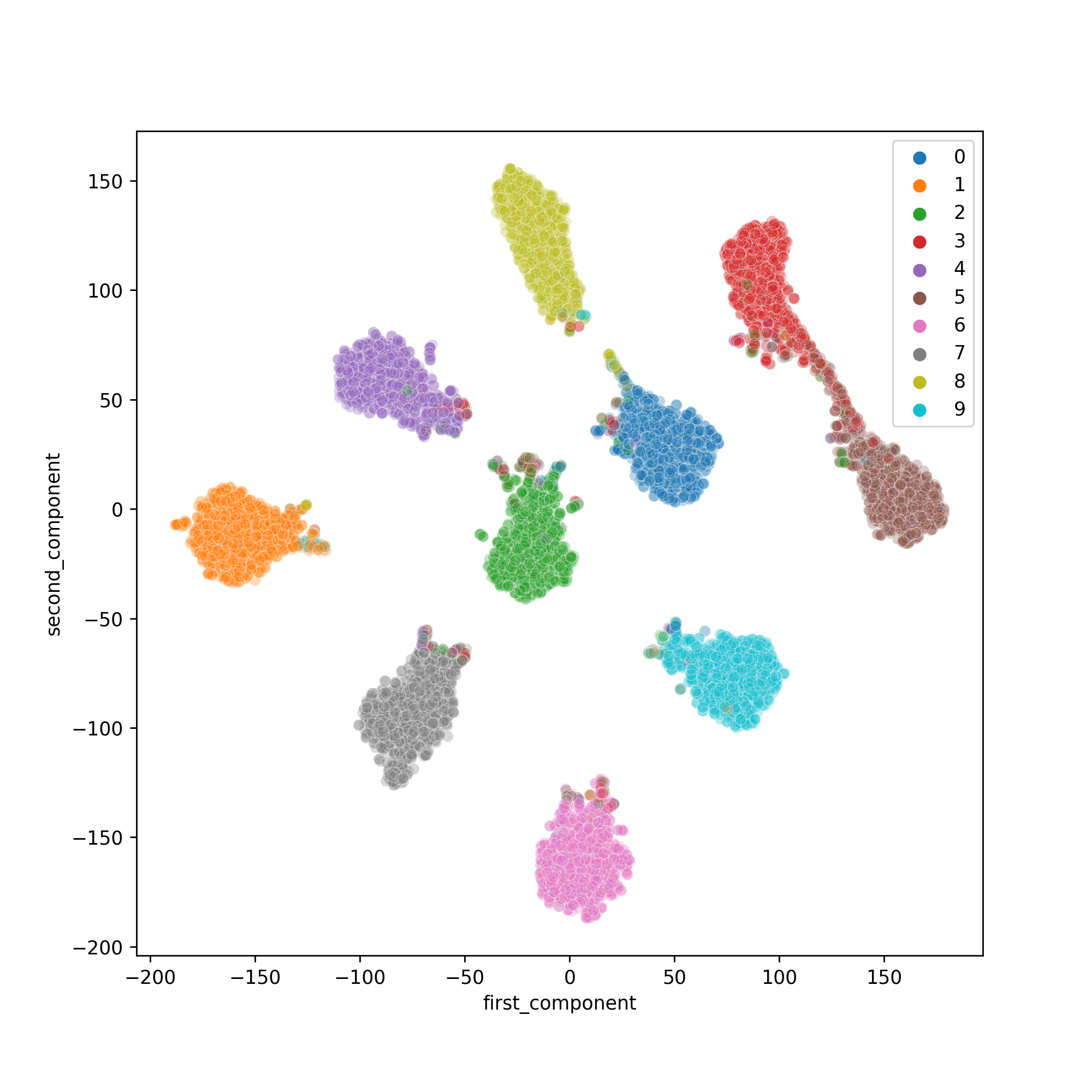}
    \caption{DBAT \textbf{features} for original examples on corresponding original classes}
  \end{subfigure}
  \hspace{0.03\textwidth}
  \begin{subfigure}{0.25\linewidth}
    \includegraphics[width=0.97\linewidth]{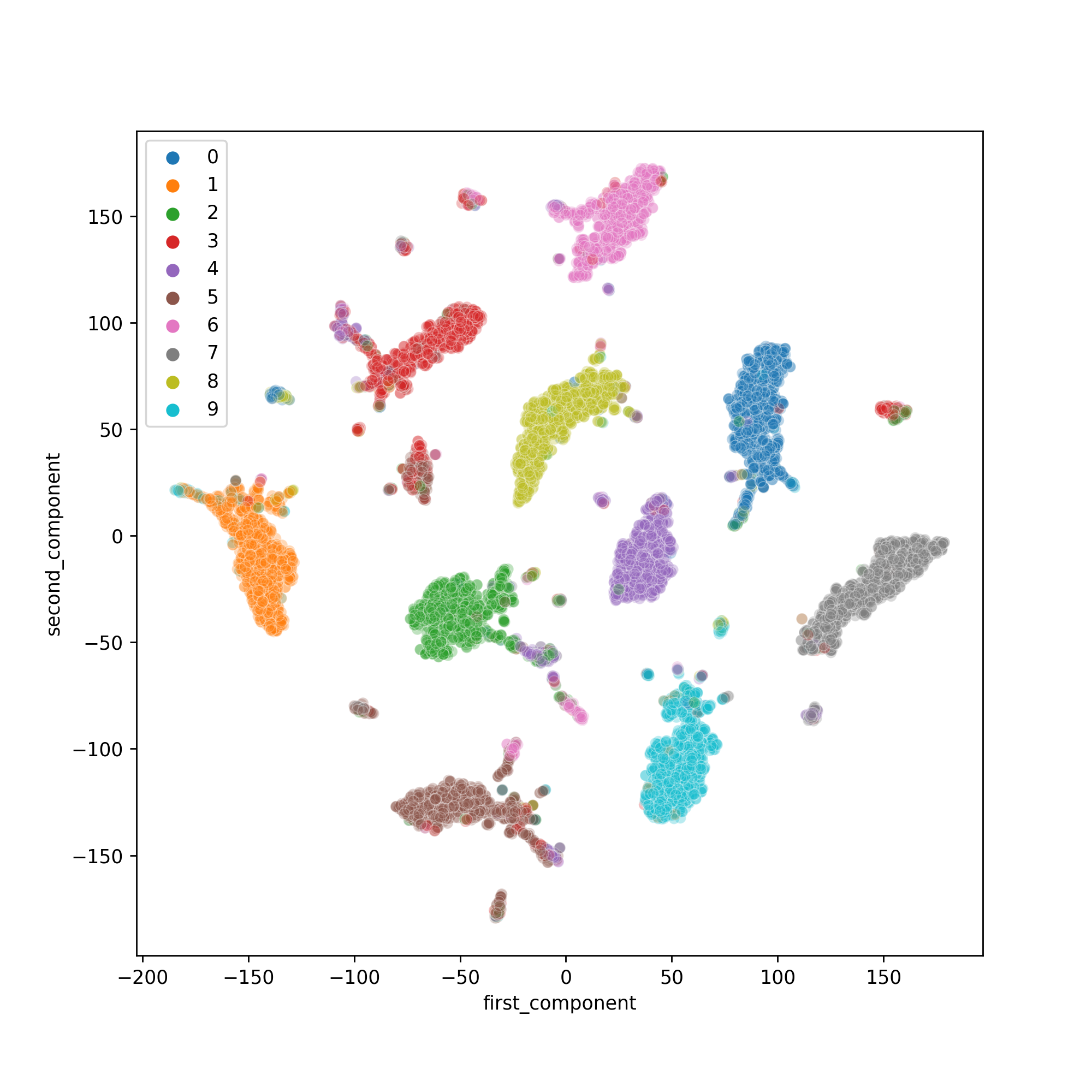}
    \caption{DBAT \textbf{features} for adv. examples on new generated adv. classes}
  \end{subfigure}
  \hspace{0.03\textwidth}
  \begin{subfigure}{0.25\linewidth}
    \includegraphics[width=0.97\linewidth]{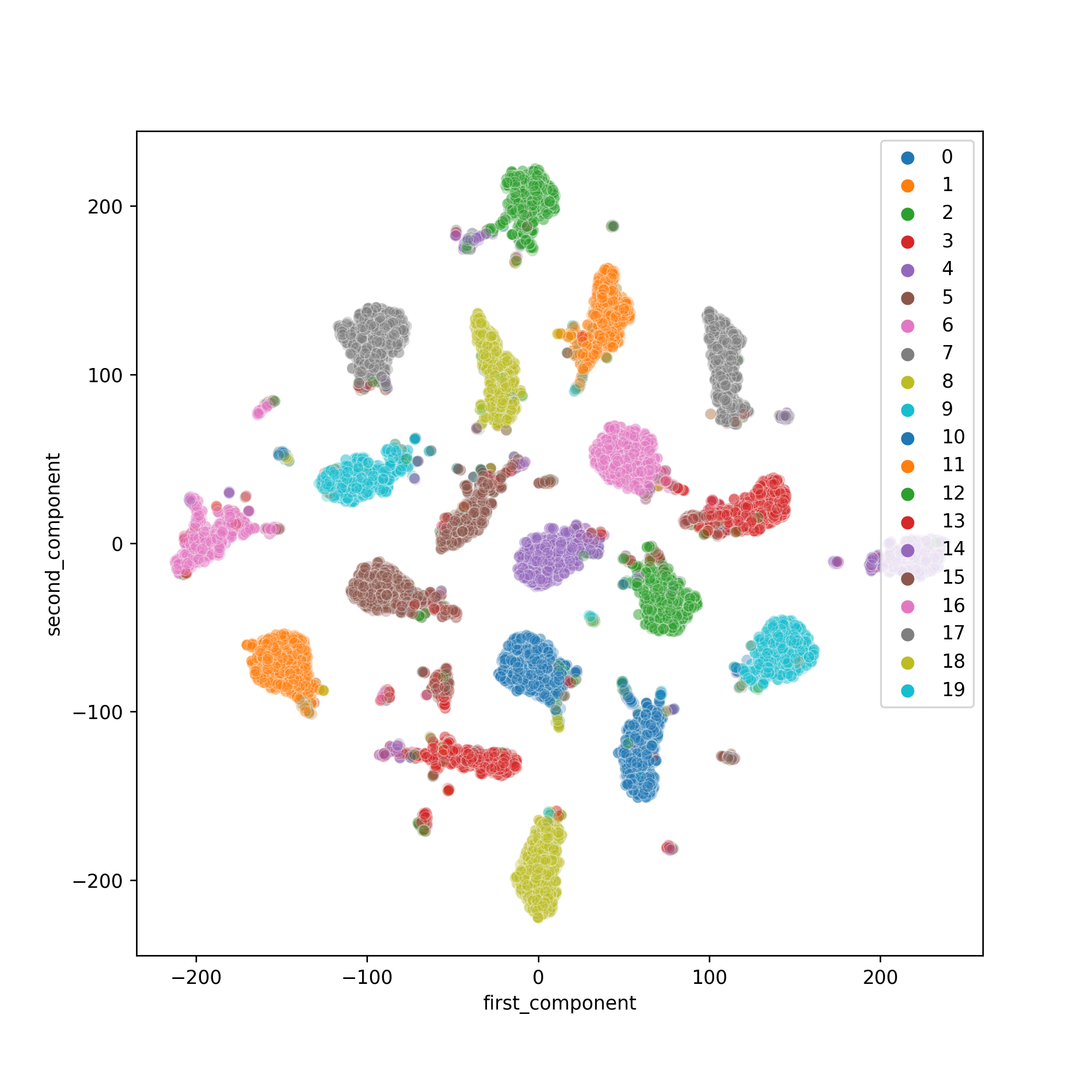}
    \caption{DBAT \textbf{features} for both clean and adversarial examples on all classes}
  \end{subfigure}
  \caption{Visualizing the \textbf{embedded feature space} of DBAT on CIFAR-10 test set using T-SNE \cite{van2008visualizing} with two components on the model output for (a) clean examples (b) adversarial examples with their new generated adversarial classes (c) combined 2-D visualization of both clean and adversarial examples with all 20 classes (same color for natural class and its adversarial counterpart). We can observe the strong separation between classes obtained by DBAT, for both original and newly generated classes.}
  \label{tsne-emb}
\end{figure*}

\section{Additional Results for Unforeseen Adversaries Robustness}
\label{unforeseen-adversaries}
In Tables \ref{l2-res}, \ref{l1-res}, \ref{l2-deepfool-res}, \ref{linf-deepfool-res}, and in \ref{cw-res}.
We used $\epsilon=0.5$ for $\ell_{2}$-PGD, $\epsilon=12$ for $\ell_{1}$-PGD, and overshoot of 0.02 for DeepFool.

\begin{table*}[ht]
  \caption{Robustness against white-box $\ell_{2}$-PGD adversary on CIFAR-10.}
  \label{l2-res}
  \centering
  \begin{tabular}{lllllll}
    \toprule
    Defense Model & Natural & PGD$^{20}$ & PGD$^{40}$ & PGD$^{60}$ & PGD$^{80}$ & PGD$^{100}$ \\
    \midrule
    DBAT (Ours) & \textbf{95.01} &  \textbf{81.74} & \textbf{79.63} & \textbf{78.34} & \textbf{77.47} & \textbf{76.45} \\
    AT & 85.10 & 76.82 & 70.33 & 66.36 & 64.57 & 63.77  \\
    TRADES & 84.92 & 76.57 & 69.74 & 65.83 & 63.72 & 62.60  \\
    LBGAT & 88.22 & 79.31 & 71.76 & 67.18 & 64.64 & 63.47 \\
    \bottomrule
  \end{tabular}
\end{table*}

\begin{table*}[ht]
  \caption{Robustness against white-box $\ell_{1}$-PGD adversary on CIFAR-10.}
  \label{l1-res}
  \centering
  \begin{tabular}{lllllll}
    \toprule
    Defense Model & Natural & PGD$^{20}$ & PGD$^{40}$ & PGD$^{60}$ & PGD$^{80}$ & PGD$^{100}$ \\
    \midrule
    DBAT (Ours) & \textbf{95.01} & \textbf{81.45} & \textbf{80.33} & \textbf{78.97} & \textbf{78.43} & \textbf{77.74} \\
    AT & 85.10 & 76.15 & 69.48 & 64.81 & 61.78 & 59.94  \\
    TRADES & 84.92 & 75.52 & 67.89 & 63.20 & 59.88 & 57.99  \\
    LBGAT & 88.22 & 78.27 & 69.87 & 64.34 & 61.00 & 58.40 \\
    \bottomrule
  \end{tabular}
\end{table*}

\begin{table*}[ht]
  \caption{Robustness against white-box $\ell_{2}$-DeepFool adversary on CIFAR-10.}
  \label{l2-deepfool-res}
  \centering
  \begin{tabular}{lllllll}
    \toprule
    Defense Model & Natural & PGD$^{2}$ & PGD$^{5}$ & PGD$^{10}$ & PGD$^{40}$ & PGD$^{100}$ \\
    \midrule
    DBAT (Ours) & \textbf{95.01} & \textbf{92.78} & \textbf{92.23} & \textbf{92.19} & \textbf{92.17} & \textbf{92.17}  \\
    AT & 85.10 & 84.44 & 84.44 & 84.44 & 84.44 & 84.44  \\
    TRADES & 84.92 & 84.23 & 84.23 & 84.23 & 84.23 & 84.23  \\
    LBGAT & 88.22 & 87.58 &  87.58 &  87.58 &  87.58 &  87.58 \\
    \bottomrule
  \end{tabular}
\end{table*}

\begin{table*}[ht]
  \caption{Robustness against white-box $\ell_{\infty}$-DeepFool adversary on CIFAR-10.}
  \label{linf-deepfool-res}
  \centering
  \begin{tabular}{lllllll}
    \toprule
    Defense Model & Natural & PGD$^{2}$ & PGD$^{5}$ & PGD$^{10}$ & PGD$^{40}$ & PGD$^{100}$ \\
    \midrule
    DBAT (Ours) & \textbf{95.01} & \textbf{83.75} & \textbf{82.12} & \textbf{81.85} & \textbf{81.72} & \textbf{81.65}  \\
    AT & 85.10 & 67.5 & 67.29 & 67.29 & 67.29 & 67.29  \\
    TRADES & 84.92 & 67.88 & 67.87 & 67.87 & 67.87 & 67.87  \\
    LBGAT & 88.22 & 71.80 & 70.89 & 70.89 & 70.89 & 70.89 \\
    \bottomrule
  \end{tabular}
\end{table*}

\begin{table*}[!ht]
  \caption{Robustness against white-box $CW_{\infty}$ adversary on CIFAR-10.}
  \label{cw-res}
  \centering
  \begin{tabular}{lllllll}
    \toprule
    Defense Model & Natural & PGD$^{20}$ & PGD$^{40}$ & PGD$^{60}$ & PGD$^{80}$ & PGD$^{100}$ \\
    \midrule
    DBAT (Ours) & \textbf{95.01} & 57.86 & 56.02 & 55.62 & 55.50 & 55.38 \\
    AT & 85.10 & 55.33 & 54.29 & 54.00 & 54.05 & 54.03  \\
    TRADES & 84.92 & 55.52 & 54.49 & 54.29 & 54.29 & 54.21  \\
    LBGAT & 88.22 & 55.34 & 54.75 & 54.12 & 54.10 & 54.10 \\
    \bottomrule
  \end{tabular}
\end{table*}

\section{Additional Results for Natural Corruptions}
\label{natural-corruption-appendix}
In Tables \ref{corruption-table1}, \ref{corruption-table2} we detail the full results of the natural corruptions test on CIFAR-10-C.

\begin{table*}[!ht]
  \caption{Accuracy against Natural Corruptions.}
  \label{corruption-table1}
  \centering
  \tiny
  \begin{tabular}{lcccccccccccccccccc}
    \toprule
    Defense Model & brightness & defocus blur & fog & glass blur & jpeg compression & motion blur & saturate & snow & speckle noise  \\
    \midrule
    DBAT & \textbf{93.06} & \textbf{91.39} & \textbf{88.53} & \textbf{81.16} & \textbf{91.38} & \textbf{88.82} & \textbf{91.79} & \textbf{88.34} & \textbf{84.81} \\
    AT &  83.30 & 80.42 & 60.22 & 77.90 & 82.73 & 76.64 & 82.31 & 80.37 & 80.74 \\
    TRADES & 82.63 & 80.04 & 60.19 & 78.00 & 82.81 & 76.49 & 81.53 & 80.68 & 80.14 \\
    LBGAT & 85.70 & 83.47 & 62.63 & 80.68 & 85.89 & 79.64 & 85.25 & 82.72 & 83.88 \\

    \bottomrule
  \end{tabular}
\end{table*}

\begin{table*}[!ht]
  \caption{Accuracy against Natural Corruptions.}
  \label{corruption-table2}
  \centering
  \tiny
  \begin{tabular}{lcccccccccccccccccc}
    \toprule
    Defense Model & contrast & elastic transform & frost & gaussian noise & impulse noise & pixelate & shot noise & spatter & zoom blur \\
    \midrule
    DBAT & \textbf{80.84} & \textbf{91.05} & \textbf{86.91} & \textbf{83.82} & \textbf{77.22} & \textbf{91.00} & \textbf{84.84} & \textbf{90.42} & \textbf{90.54} \\
    AT & 43.30 & 79.58 & 77.53 & 79.47 & 73.76 & 82.78 & 80.86 & 80.49 & 79.58 \\
    TRADES & 43.11 & 79.11 & 76.45 & 79.21 & 73.72 & 82.73 & 80.42 & 80.72 & 78.97 \\
    LBGAT & 45.65 & 82.39 & 79.71 & 82.66 & 76.27 & 85.93 & 83.98 & 83.64 & 82.57 \\
    \bottomrule
  \end{tabular}
\end{table*}

\section{Additional Results for CIFAR-100 and SVHN}
\label{cifar100-svhn-appendix}
Full numerical results for CIFAR-100 are presented in Tables \ref{white-box-cifar100-res}, \ref{black-box-cifar100-res} and in Figures \ref{all-pgd-white-box} and \ref{all-pgd-black-box}.
Full numerical results for SVHN are presented in 
Tables \ref{white-box-svhn-res}, \ref{black-box-svhn-res} 
. Attacks are generated using $\ell_{\infty}$-PGD with $\epsilon=8/255$, and perturbation step size 1/255.

\begin{table*}[!ht]
  \caption{White-box robustness on CIFAR-100.}
  \label{white-box-cifar100-res}
  \centering
  \begin{tabular}{lllllll}
    \toprule
    Defense Model & Natural & PGD$^{20}$ & PGD$^{40}$ & PGD$^{100}$ & PGD$^{200}$ & PGD$^{1000}$ \\
    \midrule
    DBAT (Ours) & \textbf{75.18} & 29.95 & 28.15 & 27.22 & 27.04 & 26.67  \\
    AT & 56.73 & 29.57 & 28.60 & 28.45 & 28.45 & 28.39 \\
    TRADES & 58.24 & 30.10 & 29.70 & 29.66 & 29.63 & 29.64 \\
    LBGAT &  60.64 & 35.80 & 34.89 & 34.84 & 34.83 & 34.79 \\
    \bottomrule
  \end{tabular}
\end{table*}

\begin{table*}[!ht]
  \caption{Black-box robustness on CIFAR-100.}
  \label{black-box-cifar100-res}
  \centering
  \begin{tabular}{lllllll}
    \toprule
    Defense Model & Natural & PGD$^{20}$ & PGD$^{40}$ & PGD$^{100}$ & PGD$^{200}$ & PGD$^{1000}$ \\
    \midrule
    DBAT (Ours) & \textbf{75.18} & \textbf{66.94} & \textbf{66.87} & \textbf{66.86} & \textbf{66.80} & \textbf{66.81}  \\
    AT & 56.73 & 55.52 & 55.43 & 55.29 & 55.20 & 55.26 \\
    TRADES & 58.24 & 57.05 & 57.03 & 56.71 & 56.84 & 56.67 \\
    LBGAT &  60.64 & 59.48 & 59.33 & 59.32 & 59.22 & 59.07  \\
    \bottomrule
  \end{tabular}
\end{table*}

\begin{table*}[!ht]
  \caption{White-box robustness on SVHN.}
  \label{white-box-svhn-res}
  \centering
  \begin{tabular}{lllllll}
    \toprule
    Defense Model & Natural & PGD$^{20}$ & PGD$^{40}$ & PGD$^{100}$ & PGD$^{200}$ & PGD$^{1000}$ \\
    \midrule
    DBAT (Ours) & \textbf{96.86} & 53.40 & 50.45 & 49.31 & 48.92 & 48.58 \\
    AT & 89.90 & 53.23 &  50.45 & 49.45 & 49.33 & 49.23 \\
    TRADES & 90.35 & 57.10 & 54.86 & 54.13 & 54.10 & 54.08 \\
    LBGAT & 91.80 & \textbf{69.12} & \textbf{66.05} & \textbf{63.38} & \textbf{61.91} & \textbf{60.13} \\
    \bottomrule
  \end{tabular}
\end{table*}

\begin{table*}[!ht]
  \caption{Black-box robustness on SVHN.}
  \label{black-box-svhn-res}
  \centering
  \begin{tabular}{lllllll}
    \toprule
    Defense Model & Natural & PGD$^{20}$ & PGD$^{40}$ & PGD$^{100}$ & PGD$^{200}$ & PGD$^{1000}$ \\
    \midrule
    DBAT (Ours) & \textbf{96.86} & \textbf{91.16} & \textbf{91.14} & \textbf{91.19} & \textbf{91.17} & \textbf{91.10} \\
    AT & 89.90 & 86.44 & 86.38 & 86.28 & 86.23 & 86.18 \\
    TRADES & 90.35 & 86.89 & 86.82 & 86.73 & 86.71 & 86.57 \\
    LBGAT & 91.80 & 88.05 & 87.90 & 87.75 & 87.70 & 87.59  \\
    \bottomrule
  \end{tabular}
\end{table*}

\section{Training with Targeted vs. Untargeted PGD}
\label{targeted-pgd}
As previously stated, using targeted-PGD on random target labels helps to better generalize the adversarial classes. In Table \ref{tpgd} 
, we show the comparison of the results between random targeted PGD (T-PGD), Least-Likely (least likely label based on the model decision) targeted PGD, and untargeted PGD. Using Least-Likely T-PGD achieves optimal natural accuracy, at the cost of a decrease in robust accuracy. Using untargeted PGD, we experience better AA robustness, at the cost of a decrease in natural accuracy.

\begin{table*}[ht]
\caption{Natural accuracy and Auto-Attack robust accuracy against the strongest adversary, \textit{Inference real-time access}, on CIFAR-10.}
\label{tpgd}
\begin{center}
\begin{small}
\begin{sc}
\begin{tabular}{lcccr}
\toprule
Attack type & Natural Acc. & AA \\
\midrule
Random T-PGD & 95.01 & 40.08 \\
\midrule
Least-Likely T-PGD & \textbf{95.67} & 34.42 \\
\midrule
Untargeted PGD & 92.18 & \textbf{44.13} \\

\bottomrule
\end{tabular}
\end{sc}
\end{small}
\end{center}
\end{table*}

Finally, we conducted a case study experiment on CIFAR-10 for two other training methods, AT \cite{madry2017towards} and TRADES \cite{zhang2019theoretically}, in order to test if random targeted PGD can help improve other methods' results as much as it helped our method. For AT, Auto-Attack robust accuracy was 45.94\%, a decrease of 5.58\% in robustness, with a small improvement of almost 3\% in natural accuracy. For TRADES, Auto-Attack robust accuracy was 52.32\%, a decrease of 0.76\% in robustness, with a small improvement of 1.07\% in natural accuracy. We can conclude that random targeted PGD degrades robustness for the other tested methods with only marginal improvement in natural accuracy.

\section{Computational Resources}
\label{compute}
Training and evaluation were done using a single NVIDIA GeForce RTX 3090, with 24 GB and GDDR6X memory. Regarding training time, DBAT introduces similar training times and resource requirements compared to the other adversarial training methods we tested.

\section{
Rademacher Analysis
}
\label{sec:rade}

\paragraph{Rademacher complexity.}
As mentioned in Section~\ref{synthetic},
the
VC-dimension is impractical in analyzing deep neural networks with a large number of weights. 
We will now argue
that the thrust of our point continues to hold for the {\em Rademacher complexity} as well, which is far more practical as far as providing finite-sample generalization bounds \cite{BartlettFT17,yin2019rademacher}. 
We assume a 
basic
familiarity
with 
this notion
and refer the reader to \cite{mohri2018foundations} for background.
For a brief recap,
if $F$ is a collection of functions mapping some set $\Omega$ to $\R$,
and $X_1,\ldots,X_n$
is sampled iid
from some distribution on $\Omega$, then the (empirical)
Rademacher complexity is defined by
\begin{equation}
R_n(F;X_1,\ldots,X_n)
=\E_\sigma\sup_{f\in F}\frac1n
\sum_{i=1}^n \sigma_i f(X_i),
\end{equation}
where 
expectation is over
the $\sigma_i$,
which are iid {\em Rademacher} variables (i.e., $
\P(\sigma_i=1)
=
\P(\sigma_i=-1)
=1/2
$).
It is a classic fact 
\cite[Theorem 3.5]{mohri2018foundations}
that the Rademacher complexity
upper-bounds the generalization error:
if $X_1,\ldots,X_n$
is an iid sample, then
\begin{equation} 
\label{eq:rade-mohri}
\gerr(h)
\le
\serr(h)
+
R_n(H;X_1,\ldots,X_n)
+
3\sqrt{\frac{\log(2/\delta)}{2n}}
\end{equation}
holds uniformly over all
$h\in H$ with probability\footnote{
where the randomness is over the sample $X_i$ 
}
at least $1-\delta$.

It is always the case
\cite[Theorem 3.2, Theorem 4.3]{devroye-combinatorial01}
that
\begin{equation}
\label{eq:rade-vc}
    R_n(H;X_1,\ldots,X_n)
    \le
    c\sqrt{\frac{V}{n}},
\end{equation}
where $V$ is the 
VC-dimension of $H$
and $c$ is a universal constant.
Observe that the left-hand side of \eqref{eq:rade-vc} is sensitive to the sampling distribution of $X_i$ (and can be arbitrarily small for very concentrated distributions),
while the right-hand side is distribution-free
---
and hence the bound
in \eqref{eq:rade-vc}
can be rather loose.
In situations where the VC-dimension is very large,
the more delicate Rademacher analysis gives much tighter bounds than VC theory
\cite{BartlettFT17,yin2019rademacher}.

We will make use of a recent
Rademacher
$\ell_{\infty}$ vector contraction
result:
\begin{theorem}\cite{foster2019ell_}\label{thm:contraction} Let $F$ be a $\R^\ell$-valued function class, such that the coordinate projection class is denoted by $F_j = \{w\mapsto f(w)_j \, | \, f\in F \}$, for $1\leq j \leq \ell$. 
Let $(\phi_t)_{t\leq n}$ be a sequence of functions such that each $\phi_t$ is $L$-Lipschitz with respect to $\ell_\infty$ norm.
For any $\alpha >0$, there exists a constant $C_{\alpha}>0$ such that if $|\phi_t(f(w))| \lor ||f(w)||_{\infty} \leq B$,
then it holds for any sequence $\mathbf{w}=(w_1,\cdots,w_n)$,
\begin{align*}
R_n(\phi \circ F  | \mathbf{w})
&:=
E_{\mathbf{\sigma}} \sup_{f\in F}\frac1n\sum_{t=1}^n\sigma_t \phi_t(f_t(w_t)) 
\\
&\leq
C_{\alpha}L\sqrt{\ell}\cdot \max_{i\in [\ell]}\sup_{\mathbf{a}=(a_1,\dots,a_n)}R_n(F_i|\mathbf{a})\cdot 
\\
&\log^{\frac{3}{2}+\alpha}\Bigg(\frac{Bn}{\max_{i\in [\ell]}\sup_{\mathbf{a}=(a_1,\dots,a_n)}R_n(F_i|\mathbf{a})}\Bigg).
\end{align*}
\end{theorem}
We observe, as in \cite[Theorem 5.1]{attias2021improved}, that $\max(x_1,\ldots,x_\ell)$ 
is
a
$1$-Lipschitz function
with respect to the $\ell_\infty$ norm. It follows (taking $\alpha=1/2$ and specializing the argument of \cite[Theorem 5.1]{attias2021improved} to our simpler case),
that if $H^{\cup \ell}$ is the {\em $\ell$-fold union} of $H$ --- that is, every $h'\in H^{\cup \ell}$ can be expressed as the union of some $\ell$ members of $H$ ---
then we have:
\begin{equation} 
\label{eq:Hcup}
    R_n(H^{\cup \ell};X_1,\ldots,X_n)
    \le
    C\sqrt \ell\max_{i\in[\ell]}\bar R_n(H)\log^2\frac{n}{\bar R_n(H)},
\end{equation}
where $C$ is a universal constant and $\bar R_n(H):=\sup_{(x_1,\ldots,x_n)\in\Omega^n}R_n(H;x_1,\ldots,x_n)$.
Since the Rademacher complexity of the $\ell$-fold union grows roughly as $\sqrt \ell$, it may often be advantageous to re-analyze complex hypotheses as
unions of simple ones --- just as we concluded for the VC-dimension in \cref{sec:theory}.
Thus, we can apply \eqref{eq:rade-mohri} to the two competing approaches: (i) when learning 
$\ell k$ ``simple'' classifiers from the class $H$,
and (ii) when learning
$k$ ``complex'' classifiers from the class $H^{\cup \ell}$.
Ignoring logarithmic factors, and treating $\delta$ as fixed,
we can compare the bound of order roughly 
$R_n(H)+\sqrt{\log(\ell k)/n}$
when splitting the classifiers
(the $\ell k$ inside the log is from the union bound)
and roughly
$\sqrt \ell R_n(H)+\sqrt{\log(k)/n}$
without splitting
(there are only $k$ classes but each incurred a $\sqrt\ell$ factor from \eqref{eq:Hcup}). 
For simple hypothesis classes $H$ (i.e., those with a low Rademacher complexity), this again demonstrates the advantage of splitting.

\mycolorgreen
\section{Distance to Decision Boundary}
\label{distance-decision}
In the following experiment, we wish to demonstrate that the conceptual illustration that we've drawn can exist in real datasets. To do so, we need to demonstrate that most clean examples are relatively close to the decision boundary, in a distance that is half of the perturbation size ($\epsilon/2=4/255$), so that it will support the way we've drawn the intuition in \ref{concept}, where examples "switch" places after the attack. Therefore, we conducted an experiment to calculate the distribution of distance from random clean examples to the decision boundary. We used SVHN with a naturally trained model using PreAct ResNet-18, and 1000 random examples from SVHN's test set. To calculate the distance, we followed the distance estimation suggested in \cite{he2018decision}. We estimate the distance
to a decision boundary in a sample of random directions in the model's input space, starting from a given input point. In each direction, we estimate the distance to a decision boundary by computing
the model's prediction on perturbed inputs at points along the direction and increase the random directions by a magnitude factor (0.002) if the prediction does not change in any of the directions. We perform this search over a set of 1,000 random orthogonal directions. Results are present in Table \ref{distance-hist}. As can be seen, the majority of the examples are very close to the decision boundary, which supports our initial intuition.

\begin{figure}[h]
\mycolorgreen
\mycolorgreen
\begin{center}
\centerline{\includegraphics[width=\columnwidth]{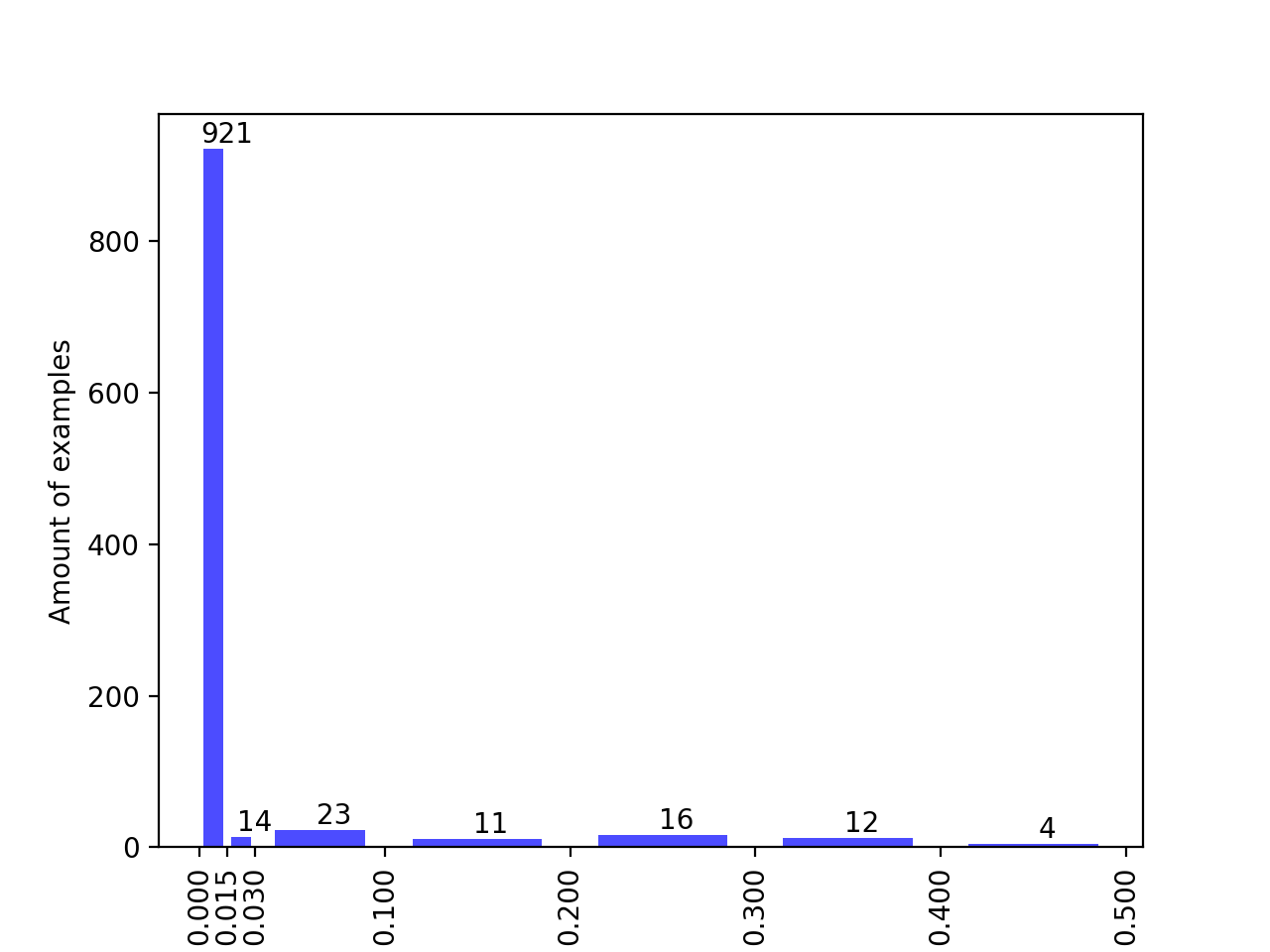}}
\caption{\mycolorgreen Histogram of the estimated distance from the decision boundary for 1000 random examples using SVHN dataset and a naturally trained model. The numbers on the different bins represent the amount of examples that fall into each bin.}
\label{distance-hist}
\end{center}
\normalcolor
\end{figure}

\normalcolor

\end{document}